\documentclass[english,letterpaper]{article}
\usepackage[T1]{fontenc}
\usepackage[latin9]{inputenc}
\usepackage{color}
\usepackage{mathtools}
\usepackage{url}
\usepackage{amsmath}
\usepackage{amsthm}
\usepackage{amssymb}
\usepackage{graphicx}
\usepackage[authoryear]{natbib}

\makeatletter

\newcommand{\lyxdot}{.}

\theoremstyle{plain}
\newtheorem{thm}{\protect\theoremname}
\theoremstyle{plain}
\newtheorem{prop}[thm]{\protect\propositionname}
\theoremstyle{plain}
\newtheorem{cor}[thm]{\protect\corollaryname}
\theoremstyle{plain}
\newtheorem{lem}[thm]{\protect\lemmaname}

\usepackage{uai2020}
\usepackage[margin=1in]{geometry}

\usepackage{times}

\usepackage{algorithmic}
\usepackage{algorithm}
\usepackage{amsmath}

\usepackage{amsthm}
\usepackage{bm}
\usepackage{grffile}
\usepackage[dvipsnames]{xcolor}
\usepackage{lipsum}

\usepackage[T1]{fontenc}

\usepackage{hyperref}
\usepackage{pgf}
\usepackage{microtype}

\usepackage{enumitem}
\usepackage{booktabs}
\usepackage{tabularx}
\usepackage{ragged2e}

\newcommand\myshade{55}
\colorlet{mylinkcolor}{violet}
\colorlet{mycitecolor}{blue!80!black}
\colorlet{myurlcolor}{Aquamarine}

\hypersetup{
  linkcolor  = mycitecolor!\myshade!black,
  citecolor  = mycitecolor!\myshade!black,
  urlcolor   = mycitecolor!\myshade!black,
  colorlinks = true,
}

\usepackage[mathscr]{eucal}

\usepackage{mathtools}

\usepackage{soul}
\usepackage{url}
\usepackage{tikz}
\usetikzlibrary{bayesnet,arrows, automata,calc}

\usepackage{thm-restate}
\usepackage{thmtools}

\allowdisplaybreaks

\@ifundefined{showcaptionsetup}{}{%
 \PassOptionsToPackage{caption=false}{subfig}}
\usepackage{subfig}
\makeatother

\usepackage{babel}
\providecommand{\corollaryname}{Corollary}
\providecommand{\lemmaname}{Lemma}
\providecommand{\propositionname}{Proposition}
\providecommand{\theoremname}{Theorem}

\begin{document}
\newcommand{\ourtitle}{Testing Goodness of Fit of Conditional Density Models with Kernels}


\ifx\arxiv
\renewcommand\Authsep{, \quad}
\renewcommand\Authands{, \quad}
\author[1]{Wittawat Jitkrittum\thanks{Contact: \texttt{wittawat@tuebingen.mpg.de}.}} 
\affil[1]{MPI for Intelligent Systems, T\"{u}bingen}
\author[2]{Heishiro Kanagawa}
\affil[2]{Gatsby Unit, University College London}
\author[1]{Bernhard Sch\"{o}lkopf}

\else
\fi

%
\author{ 
{\bf Wittawat Jitkrittum\thanks{Now with Google Research.}}  \\ 
\texttt{wittawatj@gmail.com} \\
MPI for Intelligent Systems \\ 
T\"ubingen, Germany 
\And  
{\bf Heishiro Kanagawa} \\ 
\texttt{heishiro.kanagawa@gmail.com} \\
Gatsby Unit, University College London \\ 
London, United Kingdom\\ 
\And {\bf Bernhard Sch\"{o}lkopf}   \\ 
\texttt{bs@tuebingen.mpg.de}\\
MPI for Intelligent Systems \\ 
T\"ubingen, Germany }

\global\long\def\rxess{r_{x}\mbox{-}\mathrm{ess}}%

\global\long\def\rxeq{\stackrel{r_{x}}{=}}%
\global\long\def\rxneq{\stackrel{r_{x}}{\neq}}%

\global\long\def\tph{\widehat{T_{p}^{V}}}%
\global\long\def\tpvh{\widehat{T_{p}^{V}}}%
\global\long\def\tpv{T_{p}^{V}}%

\global\long\def\hbpv{\overline{H}_{p}^{V}}%

\global\long\def\dph{\widehat{D_{p}}}%

\title{\ourtitle}
\maketitle
\begin{abstract}
We propose two nonparametric statistical tests of goodness of fit
for conditional distributions: given a conditional probability density
function $p(\mathbf{y}|\mathbf{x})$ and a joint sample, decide whether
the sample is drawn from $p(\mathbf{y}|\mathbf{x})r_{x}(\mathbf{x})$
for some density $r_{x}$. Our tests, formulated with a Stein operator,
can be applied to any differentiable conditional density model, and
require no knowledge of the normalizing constant. We show that 1)
our tests are consistent against any fixed alternative conditional
model; 2) the statistics can be estimated easily, requiring no density
estimation as an intermediate step; and 3) our second test offers
an interpretable test result providing insight on where the conditional
model does not fit well in the domain of the covariate. We demonstrate
the interpretability of our test on a task of modeling the distribution
of New York City's taxi drop-off location given a pick-up point. To
our knowledge, our work is the first to propose such conditional goodness-of-fit
tests that simultaneously have all these desirable properties.

\end{abstract}

\section{INTRODUCTION}

Conditional distributions provide a versatile tool for capturing the
relationship between a target variable and a conditioning variable
(or covariate). The last few decades has seen a broad range of modeling
applications across multiple disciplines including econometrics in
particular \citep{Mor2003,Zhe2000}, machine learning \citep{DutSalHenDei2018,UriCotGreMurLar2016},
among others. In many cases, estimating a conditional density function
from the observed data is a one of the first crucial steps in the
data analysis pipeline. While the task of conditional density estimation
has received a considerable attention in the literature, fewer works
have investigated the equally important task of evaluating the goodness
of fit of a given conditional density model.

Several approaches that address the task of conditional model evaluation
take the form of a hypothesis test. Given a conditional density model,
and a joint sample containing realizations of both target variables
and covariates, test the null hypothesis stating that the model is
correctly specified, against the alternative stating that it is not.
The model does not specify the marginal distribution of the covariates.
We refer to this task as \emph{conditional goodness-of-fit testing}.
One of the early nonparametric tests is \citet{And1997}, which extended
the classic Kolmogorov test to the conditional case. \citet{Zhe2000}
considered the first-order linear expansion of the Kullback-Leibler
divergence as the test statistic, and showed that the resulting test
is consistent against any fixed alternative under technical assumptions.
The conditional Kolmogorov test however requires estimation of the
cumulative distribution function (CDF), and may only be applied to
data of low dimension. Zheng's test involves density estimation as
part the test statistic, and test consistency is only guaranteed with
a decaying smoothing bandwidth whose rate can be challenging to control.
While there are other tests which are more computationally tractable,
these tests are only designed for conditional models from a specific
family: \citet{Mor2003} for structural equation models, \citet{StuZhu2002}
for generalized linear models, to name a few.

Another line of work which is prominent in econometrics is based on
the conditional moment restrictions (CMR). In CMR based tests, the
conditional model is specified by a conditional moment function which
has an important property that its conditional expectation under the
true data distribution is zero if and only if the model is correct.
This formulation is general, and in fact nests testing a conditional
mean regression model as a special case.
To guarantee consistency, \citet{BiePlo1997,Bie1990} use a class
of weight functions indexed by a continuous nuisance parameter so
that an infinite number of moment conditions can be considered, resulting
in a powerful test which detects any departure from the null model.
For testing the conditional mean of a regression model, the conditional
moment function can be set to the squared loss between the model output
and the target variable. However, for testing the goodness of fit
of a conditional density model, specifying the conditional moment
function is challenging, especially for a complex model whose normalizing
constant is intractable.

A related thread of development of omnibus tests for model goodness
of fit has arisen in the machine learning community recently through
the use of kernel methods and Stein operators. The combination of
Stein's identity and kernel methods was investigated in \citet{OatGirCho2017}
for the purpose of reducing the variance of Monte Carlo integration.
\citet{ChwStrGre2016,LiuLeeJor2016} independently proposed a consistent,
nonparametric test of goodness of fit of a marginal density model
known as the Kernel Stein Discrepancy (KSD) test. The KSD test has
proved successful in many applications and has spawned a number of
further studies including \citet{GorMac2017} which considered the
KSD for checking the convergence of an MCMC procedure, \citet{YanLiuRaoNev2018}
which extended the KSD test to a discrete domain, and \citet{HugMac2018,JitSzaGre2017}
which developed linear-time variants of the KSD. While proven to be
powerful, an issue with the KSD is that it is only applicable to marginal
(unconditional) density models. To our knowledge, there has been no
attempt of extending the KSD test to handle conditional density models.

In the present work, we are interested in constructing omnibus statistics
which can detect any departure from the specified conditional density
model in the null hypothesis. We propose two nonparametric, general
conditional goodness-of-fit tests which require no density estimation
as an intermediate step. Our first test, the Kernel Conditional Stein
Discrepancy (KCSD, described in Section \ref{sec:proposed_kssd}),
generalizes the KSD to conditional goodness-of-fit testing. Briefly,
we consider the KSD's Stein witness function conditioned on the covariate.
The KCSD statistic is defined as the norm, in a vector-valued reproducing
kernel Hilbert space (RKHS), of a kernel integral operator applied
to the conditional witness function. The use of the kernel integral
operator ensures that the discrepancy between the conditional model
and the data can be detected for any realization of the conditioning
variable. We prove that the KCSD test is consistent against any fixed
alternative conditional model, for any $C_{0}$-universal positive
definite kernels used; importantly, in the case of Gaussian kernels,
the consistency holds regardless of the bandwidth parameter (not necessarily
decaying in contrast to \citet{Zhe2000}).

Our second proposed test, referred to as the Finite Set Conditional
Discrepancy (FSCD, described in Section \ref{sec:proposed_fscd}),
further extends the KCSD test to also return \emph{test locations
}(a set of points) that indicate realizations of the covariate at
which the conditional model does not fit well. The FSCD test thus
offers an interpretable indication of where the conditional model
fails as evidence for rejecting the null hypothesis. Thanks to the
Stein operator, our proposed tests do not require the normalizing
constant of the conditional model. In experiments on both homoscedastic
and heteroscedastic models, we show that the KCSD test is suited for
detecting global differences, whereas the use of test locations in
the FSCD makes it more sensitive to local departure from the null
model.

\section{BACKGROUND}

\label{sec:background}This section gives background materials which
will be needed when we propose our new tests: the Kernel Conditional
Stein Discrepancy (KCSD, Section \ref{sec:proposed_kssd}) and the
Finite Set Conditional Discrepancy (FSCD, Section \ref{sec:proposed_fscd}).
We describe two known (unconditional) goodness-of-fit tests: the Kernel
Stein Discrepancy (KSD) test of \citet{ChwStrGre2016,LiuLeeJor2016}
in Section \ref{sec:ksd}, and the Finite Set Stein Discrepancy (FSSD)
of \citet{JitXuSzaFukGre2017} in Section \ref{sec:fssd}. We will
see in Sections \ref{sec:proposed_kssd} and \ref{sec:proposed_fscd}
that our proposed KCSD and FSCD are generalizations of KSD and FSSD,
respectively, to the conditional goodness-of-fit testing problem.

\subsection{KERNEL STEIN DISCREPANCY (KSD)}

\label{sec:ksd}Consider probability distributions supported on an
open subset $\mathcal{X}\subseteq\mathbb{R}^{d}$ for $d\in\mathbb{N}$.
 The Kernel Stein Discrepancy (KSD) between probability distributions
$P$ and $R$ is a divergence measure defined as $S_{P}(R)\coloneqq\sup_{\|\mathbf{f}\|_{\mathcal{F}^{d}}\le1}\left|\mathbb{E}_{\mathbf{x}\sim R}T_{P}\mathbf{f}(\mathbf{x})-\mathbb{E}_{\mathbf{x}\sim P}T_{P}\mathbf{f}(\mathbf{x})\right|$,
where $\mathbf{f}\in\mathcal{F}^{d}$, $\mathcal{F}^{d}=\times_{j=1}^{d}\mathcal{F}$,
and $\mathcal{F}$ is the reproducing kernel Hilbert space (RKHS,
\citet{BerTho2011}) associated with a positive definite kernel $k:\mathbb{R}^{d}\times\mathbb{R}^{d}\to\mathbb{R}$.

Key to the KSD is $T_{P}$, a Stein operator constructed such that
the expectation under the distribution $P$ vanishes, i.e., $\mathbb{E}_{\mathbf{x}\sim P}T_{P}\mathbf{f}(\mathbf{x})=0$,
for any function $\mathbf{f}\in\mathcal{F}^{d}$. For a distribution
$P$ admitting a differentiable, strictly positive density $p:\mathcal{X}\to(0,\infty)$,
the Langevin Stein operator of differentiable functions defined by
$T_{p}\mathbf{f}(\mathbf{x})=\mathbf{s}_{p}(\mathbf{x)}{}^{\top}\mathbf{f}(\mathbf{x})+\nabla_{\mathbf{x}}\mathbf{f}(\mathbf{x})\in\mathbb{R}^{d}$
satisfies the aforementioned condition, where $\mathbf{s}_{p}(\mathbf{x}):=\nabla_{\mathbf{x}}\log p(\mathbf{x})$
is the score function (under suitable boundary conditions \citep[Assumption A2']{OatGirCho2017}).
Thus, the KSD can be equivalently written as $\sup_{\|\mathbf{f}\|_{\mathcal{F}^{d}}\le1}\left|\mathbb{E}_{\mathbf{x}\sim R}T_{p}\mathbf{f}(\mathbf{x})\right|$
It can be shown that if the kernel $k$ is $C_{0}$-universal \citep{SriFukLan2011},
and $R$ has a density $r$ such that $\mathbb{E}_{\mathbf{x}\sim r}\lVert\nabla_{\mathbf{x}}\log p(\mathbf{x})-\nabla_{\mathbf{x}}\log r(\mathbf{x})\rVert_{2}^{2}<\infty$,
then $S_{p}(r)=0$ if and only if $p=r$ \citep[Theorem 2.2]{ChwRamSejGre2015}.

The KSD can be rewritten in a form that can be estimated easily. Assume
that the kernel $k$ is differentiable. Then, for any function $\mathbf{f}\in\mathcal{F}^{d}$,
we have $T_{p}\mathbf{f}(\mathbf{x})=\langle\mathbf{f},\xi_{p}(\mathbf{x},\cdot)\rangle_{\mathcal{F}^{d}}$
where $\xi_{p}(\mathbf{x},\cdot):=\mathbf{s}_{p}(\mathbf{x})k(\mathbf{x},\cdot)+\nabla_{x}k(\mathbf{x},\cdot)$,
due to the reproducing property of $k$, where $\langle\mathbf{f},\mathbf{g}\rangle_{\mathcal{F}^{d}}=\sum_{j=1}^{d}\langle f_{j},g_{j}\rangle_{\mathcal{F}}$
is the inner product on $\mathcal{F}^{d}$. Assuming Bochner integrability
of $\xi_{p}(\mathbf{x},\cdot)$ as in \citet{ChwStrGre2016,LiuLeeJor2016},
it follows that 
\begin{align*}
S_{p}(r) & =\sup_{\mathbf{f}\in\mathcal{F}^{d}}\left|\langle\mathbf{f},\mathbb{E}_{\mathbf{x}\sim r}\xi_{p}(\mathbf{x},\cdot)\rangle_{\mathcal{F}^{d}}\right|=\lVert\mathbf{g}_{p,r}\rVert_{\mathcal{F}^{d},}
\end{align*}
where  $\mathbf{g}_{p,r}(\cdot)=\mathbb{E}_{\mathbf{x}\sim r}\xi_{p}(\mathbf{x},\cdot)\in\mathcal{F}^{d}$
is the function that achieves the supremum, and is known as the Stein
witness function \citep{JitSzaGre2017}.  The squared KSD admits
the expression $S_{p}^{2}(r)=\lVert\mathbf{g}_{p,r}\rVert_{\mathcal{F}^{d},}^{2}=\mathbb{E}_{\mathbf{x},\mathbf{x}'\sim r}h_{p}(\mathbf{x},\mathbf{x}')$
where
\begin{align*}
h_{p}(\mathbf{x,\mathbf{x}'}): & =k(\mathbf{x},\mathbf{x}')\mathbf{s}_{p}^{\top}(\mathbf{x})\mathbf{s}_{p}(\mathbf{x}')+\sum_{i=1}^{d}\frac{\partial^{2}k(\mathbf{x},\mathbf{x}')}{\partial x_{i}\partial x_{i}'}\\
 & \phantom{=}+\mathbf{s}_{p}^{\top}(\mathbf{x})\nabla_{\mathbf{x}'}k(\mathbf{x},\mathbf{x}')+\mathbf{s}_{p}^{\top}(\mathbf{x}')\nabla_{\mathbf{x}}k(\mathbf{x},\mathbf{x}').
\end{align*}

Given a sample $\{\mathbf{x}_{i}\}_{i=1}^{n}\sim r$, the squared
KSD has an unbiased estimator $\hat{S}_{p}^{2}(r)\coloneqq\frac{1}{n(n-1)}\sum_{i\neq j}h_{p}(\mathbf{x}_{i},\mathbf{x}_{j})$,
which is a U-statistic \citep{Ser2009}. Since the KSD only depends
on $p$ through $\nabla_{\mathbf{x}}\log p(\mathbf{x})$, the normalizing
constant of $p$ is not required. The squared KSD has been successfully
used in \citet{ChwStrGre2016,LiuLeeJor2016} as the test statistic
for goodness-of-fit testing: given a marginal density model $p$ (known
up to the normalizing constant), and a sample $\{\mathbf{x}_{i}\}_{i=1}^{n}\sim r$,
test whether $p$ is the correct model.

\subsection{FINITE SET STEIN DISCREPANCY (FSSD)}

The Finite Set Stein Discrepancy (FSSD, \citet{JitSzaGre2017}) is
one of several extensions of the original KSD aiming to construct
a goodness-of-fit test of an unconditional density model that runs
in linear time (i.e., $\mathcal{O}(n)$ runtime complexity), and that
offers an interpretable test result. Key to the FSSD is the observation
that the KSD $S_{p}(r)=0$ if and only if $p=r$, assuming conditions
described in Section \ref{sec:ksd}. As a result, $\mathbf{g}_{p,r}$
is a zero function if and only if $p=r$, implying that the departure
of $\mathbf{g}_{p,r}$ from the zero function can be used to determine
whether $p$ and $r$ are the same. In contrast to the KSD which relies
on the RKHS norm $\|\cdot\|_{\mathcal{F}^{d}}$, the FSSD statistic
evaluates the Stein witness function to check this departure. Specifically,
given a finite set $V:=\{\mathbf{v}_{1},\ldots,\mathbf{v}_{J}\}\subset\mathcal{X}$
(known as the set of test locations), the squared FSSD is defined
as $\mathrm{FSSD}_{p}^{2}(r):=\frac{1}{dJ}\sum_{j=1}^{J}\|\mathbf{g}_{p,r}(\mathbf{v}_{j})\|_{2}^{2}$.
It is shown in \citet{JitXuSzaFukGre2017} that if $V$ is drawn from
a distribution with a density supported on $\mathcal{X}$, then $\mathrm{FSSD}_{p}^{2}(r)=0$
if and only if $p=r$. The squared FSSD can be estimated in linear
time, and $V$ can be optimized by maximizing the test power of the
FSSD statistic. The optimized $V$ reveals where $p$ and $r$ differ.

\label{sec:fssd}

\section{THE KERNEL CONDITIONAL STEIN DISCREPANCY (KCSD)}

\label{sec:proposed_kssd}In this section, we propose our first test
statistic called the Kernel Conditional Stein Discrepancy (KCSD) for
distinguishing two conditional probability density functions. All
omitted proofs can be found in Section \ref{sec:proofs} (appendix). 

\textbf{Problem Setting} Let $X$ and $Y$ be two random vectors taking
values in $\mathcal{X}\times\mathcal{Y}\subset\mathbb{R}^{d_{x}}\times\mathbb{R}^{d_{y}}$.
Let $p=p(\mathbf{y}|\mathbf{x})$ be a conditional density function
representing a candidate model for modeling the conditional distribution
of $\mathbf{y}$ given\textbf{ $\mathbf{x}$}.\footnote{Note that $p$ and $r$ are conditional density functions from Section
\ref{sec:proposed_kssd} onward.} Given a joint sample $Z_{n}=\{(\mathbf{x}_{i},\mathbf{y}_{i})\}_{i=1}^{n}\stackrel{i.i.d.}{\sim}r_{xy}$
where $r_{xy}(\mathbf{x},\mathbf{y})=r(\mathbf{y}|\mathbf{x})r_{x}(\mathbf{x})$
is a joint density defined on $\mathcal{X}\times\mathcal{Y}$, conditional
goodness-of-fit testing tests
\begin{align}
H_{0}\colon p & \rxeq r\text{ \ensuremath{\quad}vs\ensuremath{\quad}}H_{1}\colon p\rxneq r,\label{eq:null_hypothesis}
\end{align}
where we write $p\rxeq r$ if for $r_{x}$-almost all $\mathbf{x}$
and for all $\mathbf{y}\in\mathcal{Y}$, $p(\mathbf{y}|\mathbf{x})=r(\mathbf{y}|\mathbf{x})$.
The alternative hypothesis $H_{1}$ is the negation of $H_{0}$ and
is equivalent to the statement ``there exists a set $U\subseteq\mathcal{X}$
with $r_{x}(U)>0$ such that $p(\cdot|\mathbf{x})\neq r(\cdot|\mathbf{x})$
for all $\mathbf{x}\in U$.''  Note that $r_{xy}$ is only observed
through the joint sample $Z_{n}$; and $p$ only specifies the conditional
model. That is, $p$ does not specify a marginal model for $\mathbf{x}$.
 This subtlety is what distinguishes the conditional goodness-of-fit
testing from testing the difference between two joint distributions.

\textbf{Rationale }For machine learning applications, the proposed
null hypothesis in (\ref{eq:null_hypothesis}) allows testing the
goodness of fit of a wide range of conditional density models, including
regression models with homoscedastic or heteroscedastic noise. The
underlying prediction function can be a neural network or other arbitrarily
nonlinear functions as long as $\nabla_{\mathbf{y}}\log p(\mathbf{y}|\mathbf{x})$
is differentiable, and satisfies conditions in Theorem \ref{thm:pop_dp}.
In this work we consider $Y$ to be a continuous random vector. However,
our proposed tests can be extended to handle a discrete $Y$ to allow
testing, for instance, Bayesian classifier models $p(y|\mathbf{x})$
where $y$ represents the classification label. While the formulated
hypothesis in the current form allows testing only a fixed conditional
model (i.e., all model parameters if any must have been learned before
the test) and may appear restrictive in some cases, our goal is not
to advocate this particular null hypothesis. Rather, we see this formulation
as a first step for more realistic null hypotheses that are yet to
come; for instance, testing whether $p(\mathbf{y}|\mathbf{x},\theta)=r(\mathbf{y}|\mathbf{x})$
for some parameter vector $\theta\in\Theta$, or testing the relative
fit (with respect to the true distribution $r$) of two competing
candidate conditional models $p$ and $q$. Future tests that consider
these hypotheses can build on the results in this paper. We leave
these questions for future work.

\textbf{Vector-valued reproducing kernels} We will require vector-valued
reproducing kernels for the construction of our new tests. We briefly
give a brief introduction to this concept here. For further details,
please see Section 2.2 of \citet{CarDeToiUma2008} and \citet{CarDeToi2006,SriFukLan2011,SzaSri2018}.
Let $\mathcal{L}(\mathcal{H};\mathcal{H}')$ be the Banach space of
bounded operators from a Hilbert space $\mathcal{H}$ to $\mathcal{H}'$
endowed with the uniform norm. We write $\mathcal{L}(\mathcal{H})$
for $\mathcal{L}(\mathcal{H};\mathcal{H})$. A kernel $K\colon\mathcal{X}\times\mathcal{X}\to\mathcal{L}(\mathcal{Z})$
is said to be a $\mathcal{Z}$-reproducing kernel if $\sum_{i=1}^{N}\sum_{j=1}^{N}\left\langle K(\mathbf{x}_{i},\mathbf{x}_{j})\mathbf{z}_{i},\mathbf{z}_{j}\right\rangle _{\mathcal{Z}}\ge0$
for any $N\ge1,\{\mathbf{x}_{i}\}_{i=1}^{N}\subset\mathcal{X},\{\mathbf{z}_{i}\}_{i=1}^{N}\subset\mathcal{Z},$
and $\left\langle \diamond,\diamond\right\rangle _{\mathcal{Z}}$
denotes the inner product on $\mathcal{Z}$. Given $\mathbf{x}\in\mathcal{X}$,
we write $K_{\mathbf{x}}\colon\mathcal{Z}\to\mathcal{L}(\mathcal{X};\mathcal{Z})$
to denote the linear operator such that $K_{\mathbf{x}}\mathbf{z}\in\mathcal{L}(\mathcal{X};\mathcal{Z})$
and $(K_{\mathbf{x}}\mathbf{z})(\mathbf{t})=K(\mathbf{x},\mathbf{t})\mathbf{z}\in\mathcal{Z}$,
for all $\mathbf{x},\mathbf{t}\in\mathcal{X}$ and all $\mathbf{z}\in\mathcal{Z}$.
As in the case of a real-valued reproducing kernel, given a $\mathcal{Z}$-reproducing
kernel $K$, there exists a unique reproducing kernel Hilbert space
(RKHS) $\mathcal{F}_{K}$ such that $K_{\mathbf{x}}\in\mathcal{L}(\mathcal{Z};\mathcal{F}_{K})$
and $f(\mathbf{x})=K_{\mathbf{x}}^{*}f$ (the reproducing property)
for all $\mathbf{x}\in\mathcal{X},f\in\mathcal{F}_{K}$ and $K_{\mathbf{x}}^{*}\colon\mathcal{F}_{K}\to\mathcal{Z}$
denotes the adjoint operator of $K_{\mathbf{x}}$.

Let $\mathcal{C}(\mathcal{X};\mathcal{Z})$ be the vector space of
continuous functions mapping from $\mathcal{X}$ to $\mathcal{Z}$.
In this work, we will assume that $\mathcal{X}$ and $\mathcal{Z}$
are Banach spaces. Let $\mathcal{C}_{0}(\mathcal{X};\mathcal{Z})\subset\mathcal{C}(\mathcal{X};\mathcal{Z})$
denote the subspace of continuous functions that vanish at infinity
i.e., $\|f(\mathbf{x})\|_{\mathcal{Z}}\to0$ as $\|\mathbf{x}\|\to\infty$.
A $\mathcal{Z}$-reproducing kernel $K\colon\mathcal{X}\times\mathcal{X}\to\mathcal{L}(\mathcal{Z})$
is said to be $C_{0}$ if $\mathcal{F}_{K}$ is a subspace of $\mathcal{C}_{0}(\mathcal{X};\mathcal{Z})$
\citep[Section 2.3, Definition 1]{CarDeToiUma2008}. A $C_{0}$-kernel
$K$ is said to be \emph{universal} if $\mathcal{F}_{K}$ is dense
in $L^{2}(\mathcal{X},\mu;\mathcal{Z})$ for any probability measure
$\mu$ \citep[Section 4.1]{CarDeToiUma2008}.

Let $l\colon\mathcal{Y}\times\mathcal{Y}\to\mathbb{R}$ be a positive
definite kernel associated with the RKHS $\mathcal{F}_{l}$. Write
$\mathcal{F}_{l}^{d_{y}}:=\times_{i=1}^{d_{y}}\mathcal{F}_{l}$ and
define $\left\langle \mathbf{a},\mathbf{b}\right\rangle _{\mathcal{F}_{l}^{d_{y}}}:=\sum_{i=1}^{d_{y}}\left\langle a_{i},b_{i}\right\rangle _{\mathcal{F}_{l}}$
to be the inner product on $\mathcal{F}_{l}^{d_{y}}$ for $\mathbf{a}:=(a_{1},\ldots,a_{d_{y}}),\mathbf{b}:=(b_{1},\ldots,b_{d_{y}})\in\mathcal{F}_{l}^{d_{y}}$.
Let $K\colon\mathcal{X}\times\mathcal{X}\to\mathcal{F}_{l}^{d_{y}}$
be a $\mathcal{F}_{l}^{d_{y}}$-reproducing kernel i.e., $\mathcal{Z}=\mathcal{F}_{l}^{d_{y}}$.
Let $k\colon\mathcal{X}\times\mathcal{X}\to\mathbb{R}$ be a real-valued
kernel associated with the RKHS $\mathcal{F}_{k}$. For brevity, we
write $\mathbb{E}_{\mathbf{xy}}$ for $\mathbb{E}_{(\mathbf{x},\mathbf{y})\sim r_{xy}}$.
In what follows, we will interchangeably write $p_{|\mathbf{x}}$
and $p(\cdot|\mathbf{x})$.

\textbf{Proposed statistic} Consider the following population statistic
defining a discrepancy between $p$ and $r$: 
\begin{align}
D_{p}(r): & =\big\|\mathbb{E}_{(\mathbf{x},\mathbf{y})\sim r_{xy}}K_{\mathbf{x}}\xi_{p_{|\mathbf{x}}}(\mathbf{y},\diamond)\big\|_{\mathcal{F}_{K}}^{2},\label{eq:kssd_pop}
\end{align}
where $\xi_{p_{|\mathbf{x}}}(\mathbf{y},\cdot):=l(\mathbf{y},\cdot)\nabla_{\mathbf{y}}\log p(\mathbf{y}|\mathbf{x})+\nabla_{\mathbf{y}}l(\mathbf{y},\cdot)\in\mathcal{F}_{l}^{d_{y}}$.
We refer to $D_{p}(r)$ as the Kernel Conditional Stein Discrepancy
(KCSD). Our first result in Theorem \ref{thm:pop_dp} shows that the
KCSD is zero if and only if $p\stackrel{r_{x}}{=}r$.
\begin{restatable}[$D_{p}(r)$ distinguishes conditional density functions]{thm}{popstatdiv}

\label{thm:pop_dp}Let $K\colon\mathcal{X}\times\mathcal{X}\to\mathcal{L}(\mathcal{F}_{l}^{d_{y}})$
and $l\colon\mathcal{Y}\times\mathcal{Y}\to\mathbb{R}$ be positive
definite kernels. Define $\mathbf{g}_{p,r}(\mathbf{w}|\mathbf{x}):=\mathbb{E}_{\mathbf{y}\sim r_{|\mathbf{x}}}\xi_{p_{|\mathbf{x}}}(\mathbf{y},\mathbf{w})\in\mathbb{R}^{d_{y}}$
where $\mathbf{g}_{p,r}(\cdot|\mathbf{x})\in\mathcal{F}_{l}^{d_{y}}$
for each $\mathbf{x}$. Assume that
\begin{enumerate}
\item $K$ and $l$ are $C_{0}$-universal;
\item $\rxess\sup_{\mathbf{x}}\mathbb{E}_{\mathbf{y}\sim r(\mathbf{y}|\mathbf{x})}\big\|\nabla_{\mathbf{y}}\log\frac{p(\mathbf{y}|\mathbf{x})}{r(\mathbf{y}|\mathbf{x})}\big\|_{2}^{2}<\infty$;
\item $\int_{\mathcal{X}}\|\mathbf{g}_{p,r}(\diamond|\mathbf{x})\|_{\mathcal{F}_{l}^{d_{y}}}^{2}r_{x}(\mathbf{x})\thinspace\mathrm{d}\mathbf{x}<\infty$.
\item \label{enu:assume_bochner}$\mathbb{E}_{\mathbf{xy}}\|K_{\mathbf{x}}\xi_{p_{|\mathbf{x}}}(\mathbf{y},\diamond)\|_{\mathcal{F}_{K}}<\infty$;\textcolor{red}{}
\end{enumerate}
Then $D_{p}(r)=0$ if and only if $p\rxeq r$ i.e., for $r_{x}$-almost
all $\mathbf{x}\in\mathcal{X}$, $p(\cdot|\mathbf{x})=r(\cdot|\mathbf{x})$.
\end{restatable}

\begin{proof}[Proof (sketch)]
The idea is to rewrite (\ref{eq:kssd_pop}) into a form that involves
the Stein witness function (as described in Section \ref{sec:background})
$\mathbf{g}_{p,r}(\diamond|\mathbf{x})$ between $p(\cdot|\mathbf{x})$
and $r(\cdot|\mathbf{x})$. It then amounts to showing that $\mathbf{g}_{p,r}(\diamond|\mathbf{x})$
is a zero function for $r_{x}$-almost all $\mathbf{x}$. This is
done by applying the integral operator $\mathbf{f}_{\mathbf{x}}\mapsto\int K_{\mathbf{x}}\mathbf{f}_{\mathbf{x}}r_{x}(\mathbf{x})\thinspace\mathrm{d}\mathbf{x}$
on $\mathbf{g}_{p,r}(\diamond|\mathbf{x})$ to incorporate ($r_{x}$-almost)
all $\mathbf{x}$. The result is $G_{p,r}=\int K_{\mathbf{x}}\mathbf{g}_{p,r}(\diamond|\mathbf{x})r_{x}(\mathbf{x})\thinspace\mathrm{d}\mathbf{x}$.
Since $K$ is $C_{0}$-universal, this operator is injective, implying
$G_{p,r}$ is zero if and only if $\mathbf{g}_{p,r}(\diamond|\mathbf{x})$
is a zero function for $r_{x}$-almost all $\mathbf{x}$. But, $G_{p,r}=\mathbb{E}_{(\mathbf{x},\mathbf{y})\sim r_{xy}}K_{\mathbf{x}}\xi_{p_{|\mathbf{x}}}(\mathbf{y},\diamond)$.
Thus, taking the norm gives (\ref{eq:kssd_pop}). See Section \ref{subsec:proof_popstatdiv}
for the complete proof.
\end{proof}
In the proof sketch, we can see the application of the integral operator
$\mathbf{f}_{\mathbf{x}}\mapsto\int K_{\mathbf{x}}\mathbf{f}_{\mathbf{x}}r_{x}(\mathbf{x})\thinspace\mathrm{d}\mathbf{x}$
as taking into account the conditional Stein witness function $\mathbf{g}_{p,r}(\diamond|\mathbf{x})$
of ($r_{x}$-almost) all $\mathbf{x}$ at the same time. Theorem \ref{thm:pop_dp}
states that the population statistic in (\ref{eq:kssd_pop}) distinguishes
two conditional density functions under regularity conditions given
above. In particular, it is required that the two kernels $K$ and
$l$ are $C_{0}$-universal. Examples of a real-valued $C_{0}$-universal
kernels are the Gaussian kernel $l(\mathbf{y},\mathbf{y}'):=\exp\left(-\frac{\|\mathbf{y}-\mathbf{y}'\|_{2}^{2}}{2\sigma_{y}^{2}}\right)\in\mathbb{R}$,
Laplace kernel, and the inverse multiquadrics kernel \citep[p. 2397]{SriFukLan2011}.
An example of a $\mathcal{F}_{l}^{d_{y}}$-reproducing, $C_{0}$-universal
kernel $K$ is $K(\mathbf{x},\mathbf{x}')=k(\mathbf{x},\mathbf{x}')I$
where $k$ is a real-valued $C_{0}$-universal kernel, and $I\in\mathcal{L}(\mathcal{F}_{l}^{d_{y}})$
is the identity operator \citep[Example 14]{CarDeToiUma2008}. For
simplicity, in this work, we will assume a kernel $K$ that takes
this form. 

\subsection{HYPOTHESIS TESTING WITH KCSD}

\label{sec:kssd_test}To construct a statistical test for conditional
goodness of fit, we start by rewriting $D_{p}(r)$ in (\ref{eq:kssd_pop})
in a form that can be estimated easily as shown in Proposition \ref{prop:kssd_pop_ustat}. 
\begin{prop}
\label{prop:kssd_pop_ustat}Assume that $K(\mathbf{x},\mathbf{x}'):=k(\mathbf{x},\mathbf{x}')I$
for a positive definite kernel $k\colon\mathcal{X}\times\mathcal{X}\to\mathbb{R}$.
Define $\mathbf{s}_{p}(\mathbf{y}|\mathbf{x}):=\nabla_{\mathbf{y}}\log p(\mathbf{y}|\mathbf{x})$.
Then, 
\begin{equation}
D_{p}(r)=\mathbb{E}_{\mathbf{xy}}\mathbb{E}_{\mathbf{x'y'}}k(\mathbf{x},\mathbf{x}')h_{p}((\mathbf{x},\mathbf{y}),(\mathbf{x}',\mathbf{y}')),\label{eq:kssd_pop_ustat}
\end{equation}
where $h_{p}((\mathbf{x},\mathbf{y}),(\mathbf{x}',\mathbf{y}'))$\vspace{-3mm}
\begin{align}
 & :=l(\mathbf{y},\mathbf{y}')\mathbf{s}_{p}^{\top}(\mathbf{y}|\mathbf{x})\mathbf{s}_{p}(\mathbf{y}'|\mathbf{x}')+\sum_{i=1}^{d_{y}}\frac{\partial^{2}}{\partial y_{i}\partial y_{i}'}l(\mathbf{y},\mathbf{y}')\nonumber \\
 & \phantom{:=}+\mathbf{s}_{p}^{\top}(\mathbf{y}|\mathbf{x})\nabla_{\mathbf{y}'}l(\mathbf{y},\mathbf{y}')+\mathbf{s}_{p}^{\top}(\mathbf{y}'|\mathbf{x}')\nabla_{\mathbf{y}}l(\mathbf{y},\mathbf{y}'),\label{eq:unsmoothed_ustat_kernel}
\end{align}
\end{prop}

Define $H_{p}((\mathbf{x},\mathbf{y}),(\mathbf{x}',\mathbf{y}')):=k(\mathbf{x},\mathbf{x}')h_{p}((\mathbf{x},\mathbf{y}),(\mathbf{x}',\mathbf{y}'))$.
Given an i.i.d. sample $\{(\mathbf{x}_{i},\mathbf{y}_{i})\}_{i=1}^{n}\sim r_{xy}$.
an unbiased, consistent estimator for (\ref{eq:kssd_pop_ustat}) is
given by 
\begin{align}
\widehat{D_{p}}:= & \frac{1}{n(n-1)}\sum_{i\neq j}H_{p}((\mathbf{x}_{i},\mathbf{y}_{i}),(\mathbf{x}_{j},\mathbf{y}_{j})),\label{eq:kssd_emp_ustat}
\end{align}
which is a second-order U-statistic with $H_{p}$ as the U-statistic
kernel \citep[Section 5]{Ser2009}, and can be computed easily. It
is clear from (\ref{eq:unsmoothed_ustat_kernel}) that the KCSD statistic
(both population and its estimator) depends on the model $p$ only
through $\nabla_{\mathbf{y}}\log p(\mathbf{y}|\mathbf{x})=\nabla_{\mathbf{y}}\log p(\mathbf{y},\mathbf{x})$
which is independent of the normalizer $p(\mathbf{x})$. The fact
that the KCSD does not require the normalizer is a big advantage since
modern conditional models tend to be complex and their normalizers
may not be tractable. A consequence of being a U-statistic is that
its asymptotic behaviors can be derived straightforwardly, as given
in Proposition \ref{prop:asymptotic_kssd}. 
\begin{prop}[Asymptotic distributions of $\dph$]
\label{prop:asymptotic_kssd} Assume all conditions in Theorem \ref{thm:pop_dp}
and assume that $\mathbb{E}_{\mathbf{xy}}\mathbb{E}_{\mathbf{x'y'}}H_{p}^{2}((\mathbf{x},\mathbf{y}),(\mathbf{x}',\mathbf{y}'))<\infty$.
Then,
\begin{enumerate}
\item Under $H_{0}$, $n\dph\stackrel{d}{\to}\sum_{j=1}^{\infty}\lambda_{j}(\chi_{j1}^{2}-1)$,
where $\{\chi_{1j}^{2}\}_{j}$ are independent $\chi_{1}^{2}$ random
variables, $\lambda_{j}$ are eigenvalues of the operator $A$ defined
as $(A\varphi)(\mathbf{z})=\int H_{p}(\mathbf{z},\mathbf{z}')\varphi(\mathbf{z}')r_{xy}(\mathbf{z}')\thinspace\mathrm{d}\mathbf{z}'$
for non-zero $\varphi$, $\mathbf{z}:=(\mathbf{x},\mathbf{y})$ and
$\mathbf{z}':=(\mathbf{x}',\mathbf{y}')$;
\item Under $H_{1}$, $\sqrt{n}\left(\dph-D_{p}(r)\right)\stackrel{d}{\to}\mathcal{N}(0,\sigma_{H_{1}}^{2})$
where $\sigma_{H_{1}}^{2}:=4\mathbb{V}[\mathbb{E}_{\mathbf{xy}}[H_{p}((\mathbf{x},\mathbf{y}),(\mathbf{x}',\mathbf{y}'))]]$.
\end{enumerate}
\end{prop}

A proof of Proposition \ref{prop:asymptotic_kssd} can be found in
Section \ref{subsec:proof_asymp_kssd} (appendix). Proposition \ref{prop:asymptotic_kssd}
suggests that under $H_{0}$, $n\dph$ converges to a limit distribution
given by an infinite weighted sum of chi-squared random variables.
Under $H_{1}$, for any fixed $p$ and $r$, we have $n\dph=\mathcal{O}_{p}(\sqrt{n})$,
which diverges to $+\infty$, and allows the test to reject $H_{0}$
when $n$ is sufficiently large. The behaviors are common in many
recently developed nonparametric tests \citep{YanLiuRaoNev2018,ChwStrGre2016,LiuLeeJor2016,GreFukTeoSonSch2008,GreBorRasSchSmo2012}.
A consistent test that has an asymptotic false rejection rate no larger
than a specified significance level $\alpha\in(0,1)$ can be constructed
by setting the rejection threshold (critical value) to be $\gamma_{1-\alpha}=(1-\alpha)$-quantile
of the asymptotic null distribution. That is, the test rejects the
null hypothesis $H_{0}$ if $n\dph>\gamma_{1-\alpha}$. In practice
however, the limiting distribution under $H_{0}$ is not available
in closed form, and we have to resort to approximating the test threshold
either by bootstrapping \citep{ArcGin1992,HusJan1993} or estimating
the eigenvalues $\{\lambda_{j}\}_{j}$ which can cost $\mathcal{O}(n^{3})$
runtime \citep{GreFukHarSri2009}.

\textbf{Test threshold} In our work, we use the bootstrap procedure
of \citet{ArcGin1992,HusJan1993} as also used in the KSD test of
\citet{LiuLeeJor2016,YanLiuRaoNev2018} (with a U-statistic estimator)
and \citet{ChwRamSejGre2015} (with a V-statistic estimator). To generate
a bootstrap sample, we draw $w_{1},\ldots,w_{n}\sim\mathrm{Multinomial}\left(n;\frac{1}{n},\ldots,\frac{1}{n}\right)$,
define $\tilde{w}_{i}:=\frac{1}{n}(w_{i}-1)$, and compute $\dph^{*}=\sum_{i=1}^{n}\sum_{j\neq i}\tilde{w}_{i}\tilde{w}_{j}H_{p}((\mathbf{x}_{i},\mathbf{y}_{i}),(\mathbf{x}_{i},\mathbf{y}_{j})).$
By bootstrapping $m$ times to generate $\dph_{1}^{*},\ldots\dph_{m}^{*}$,
the test threshold can be estimated by computing the empirical $(1-\alpha)$-quantile
of these bootstrapped samples. The overall computational cost of this
bootstrap procedure is $\mathcal{O}(mn^{2})$, which is the same cost
as testing a marginal probability model in the KSD test. 

\section{THE FINITE SET CONDITIONAL DISCREPANCY (FSCD)}

\label{sec:proposed_fscd}In this section, we extend the KCSD statistic
presented in Section \ref{sec:proposed_kssd} to enable it to also
pinpoint the location(s) in the domain of $\mathcal{X}$ that best
distinguish $p(\cdot|\mathbf{x})$ and $r(\cdot|\mathbf{x})$. The
result is a goodness-of-fit test for conditional density functions
which gives an interpretable output (in terms of locations in $\mathcal{X})$
to justify a rejection of the null hypothesis.

We start by noting that Theorem \ref{thm:pop_dp} and (\ref{eq:kssd_pop})
implies that $G_{p,r}\colon\mathcal{X}\to\mathcal{F}_{l}^{d_{y}}$
defined as $G_{p,r}(\mathbf{v}):=\left[\mathbb{E}_{(\mathbf{x},\mathbf{y})\sim r_{xy}}K_{\mathbf{x}}\xi_{p_{|\mathbf{x}}}(\mathbf{y},\diamond)\right](\mathbf{v})\in\mathcal{F}_{l}^{d_{y}}$
is a zero function if and only if $p\rxeq r$, under the conditions
described in the theorem statement. Note that the KCSD $D_{p}(r)=\|G_{p,r}\|_{\mathcal{F}_{K}}^{2}$.
For a fixed $\mathbf{v}\in\mathcal{X}$, the function $\mathbf{v}\mapsto\frac{1}{d_{y}}\|G(\mathbf{v})\|_{\mathcal{F}_{l}^{d_{y}}}^{2}\ge0$
can be seen as quantifying the extent to which $p$ and $r$ differ,
as measured at $\mathbf{v}\in\mathcal{X}$; that is, the higher $\frac{1}{d_{y}}\|G(\mathbf{v})\|_{\mathcal{F}_{l}^{d_{y}}}^{2}$,
the larger the discrepancy between $p(\cdot|\mathbf{v})$ and $r(\cdot|\mathbf{v})$.
Inspired by \citet{JitXuSzaFukGre2017}, one can thus construct a
variant of the KCSD statistic as follows. Given a set of $J$ \emph{test
locations} $V:=\{\mathbf{v}_{i}\}_{i=1}^{J}\subset\mathcal{X}$, we
evaluate $G_{p,r}(\mathbf{v})$ at these locations instead of taking
the norm $\|\cdot\|_{\mathcal{F}_{K}}$ \citep{JitSzaChwGre2016,JitSzaGre2017,JitXuSzaFukGre2017,SceVar2019}.
More formally, we propose a statistic defined as 
\begin{equation}
\tpv(r):=\frac{1}{Jd_{y}}\sum_{i=1}^{J}\|G_{p,r}(\mathbf{v}_{i})\|_{\mathcal{F}_{l}^{d_{y}}}^{2},\label{eq:fscd_pop}
\end{equation}
which we refer to as the \emph{Finite Set Conditional Discrepancy
(FSCD)}. Later in Section \ref{sec:optimized_fscd}, we will describe
how $V$ can be automatically optimized by maximizing the test power
of the FSCD test. The optimized test locations in $V$ are interpretable
in the sense that they specify points $\{\mathbf{v}_{i}\}_{i=1}^{J}$
in $\mathcal{X}$ that best reveal the differences between the two
conditional density functions. For the purpose of describing the statistic
in this section, we assume that $V$ is given. We first show in Theorem
\ref{thm:fscd_pop} that the FSCD almost surely distinguishes two
conditional probability density functions.

\begin{restatable}{thm}{tpvstatdiv}

\label{thm:fscd_pop}Assume all conditions in Theorem \ref{thm:pop_dp}.
Further assume that $\mathcal{X}\subseteq\mathbb{R}^{d_{x}}$ is a
connected open set, and $K(\mathbf{x},\mathbf{x}')=k(\mathbf{x},\mathbf{x}')I$
where $k\colon\mathcal{X}\times\mathcal{X}\to\mathbb{R}$ is a real
analytic kernel i.e., for any $\mathbf{x}\in\mathcal{X}$, $\mathbf{v}\mapsto k(\mathbf{x},\mathbf{v})$
is a real analytic function. Then, for any $J\in\mathbb{N}$, the
following statements hold:
\begin{enumerate}
\item Under $H_{0}$, $\tpv(r)=0$ for any $V=\{\mathbf{v}_{j}\}_{j=1}^{J}\subset\mathcal{X}$.
\item Under $H_{1}$, if $\mathbf{v}_{1},\ldots,\mathbf{v}_{J}$ in $V$
are drawn from a probability density $\eta$ whose support is $\mathcal{X}$,
then $\eta$-almost surely $T_{p}^{V}(r)>0$.
\end{enumerate}
\end{restatable}

Theorem \ref{thm:fscd_pop} states that given $p$ and $r$, $\tpv(r)=0$
if and only if $p\rxeq r$ when $V$ is drawn from \emph{any }probability
density supported on $\mathcal{X}$. The core idea is that $\|G_{p,r}(\mathbf{v})\|_{\mathcal{F}_{l}^{d_{y}}}^{2}$
is a real analytic function of $\mathbf{v}$ if $k$ is a real analytic
kernel. It is known that the set of roots of a non-zero real analytic
function has zero Lebesgue measure \citep{Mit2015}. So, pointwise
evaluations at the $J$ random test locations suffice to check whether
$G_{p,r}$ is a zero function, and the result follows. The FSCD statistic
in (\ref{eq:fscd_pop}) can thus be seen as quantifying the average
discrepancy between $p(\cdot|\mathbf{x})$ and $r(\cdot|\mathbf{x})$
as measured at the locations $\mathbf{x}\in V$.

\subsection{HYPOTHESIS TESTING WITH FSCD}

To perform hypothesis testing the FSCD, we first show in Proposition
\ref{prop:fscd_pop_ustat} that $\tpv(r)$ in (\ref{eq:fscd_pop})
can be written as a U-statistic. 
\begin{prop}
\label{prop:fscd_pop_ustat} Given a set of test locations $V=\{\mathbf{v}_{j}\}_{j=1}^{J}\subset\mathcal{X}$,
in (\ref{eq:fscd_pop}), $\|G_{p,r}(\mathbf{v})\|_{\mathcal{F}_{l}^{d_{y}}}^{2}=\mathbb{E}_{\mathbf{x}\mathbf{y}}\mathbb{E}_{\mathbf{x'y'}}k(\mathbf{x},\mathbf{v})k(\mathbf{x}',\mathbf{v})h_{p}((\mathbf{x},\mathbf{y}),(\mathbf{x}',\mathbf{y}'))$
($h_{p}$ is defined in (\ref{eq:unsmoothed_ustat_kernel})) and 
\begin{equation}
\tpv(r)=\mathbb{E}_{\mathbf{x}\mathbf{y}}\mathbb{E}_{\mathbf{x'y'}}\hbpv((\mathbf{x},\mathbf{y}),(\mathbf{x}',\mathbf{y}')),\label{eq:fscd_pop_ustat}
\end{equation}
 where $\hbpv((\mathbf{x},\mathbf{y}),(\mathbf{x}',\mathbf{y}')):=\frac{1}{d_{y}}\overline{k}_{V}(\mathbf{x},\mathbf{x}')h_{p}((\mathbf{x},\mathbf{y}),(\mathbf{x}',\mathbf{y}'))$
and $\overline{k}_{V}(\mathbf{x},\mathbf{x}'):=\frac{1}{J}\sum_{i=1}^{J}k(\mathbf{x},\mathbf{v}_{i})k(\mathbf{x}',\mathbf{v}_{i})$.
\end{prop}

Similarly to (\ref{eq:kssd_emp_ustat}), an unbiased estimator of
$\tpv$ is given by a second-order U-statistic: $\tph:=\frac{1}{n(n-1)}\sum_{i\neq j}\hbpv((\mathbf{x}_{i},\mathbf{y}_{i}),(\mathbf{x}_{j},\mathbf{y}_{j})).$
It is clear from (\ref{eq:fscd_pop_ustat}) and the definition of
$\hbpv$ that the FSCD statistic is in fact a special case of the
KCSD with the kernel $k$ in (\ref{eq:kssd_pop_ustat}) replaced with
$\overline{k}_{V}$. For this reason, the asymptotic distributions
of $\tph$ under both $H_{0}$ and $H_{1}$ are almost identical to
those of the KCSD. We omit the result here and present it in Proposition
\ref{prop:fscd_asymptotics} in the appendix. Since $\tph$ is also
a degenerate U-statistic, the test threshold can be obtained by bootstrapping
with weights drawn from the multinomial distribution as in the case
of the KCSD.

\subsection{OPTIMIZING TEST LOCATIONS}

\label{sec:optimized_fscd}While Theorem \ref{thm:fscd_pop} guarantees
that the FSCD can distinguish two conditional density functions with
any $V$ drawn from any probability density supported on $\mathcal{X}$,
in practice, optimizing $V$ will further increase the power of the
test, and allow us to interpret $V$ as the locations in $\mathcal{X}$
for which the difference between $p(\cdot|\mathbf{x})$ and $r(\cdot|\mathbf{x})$
can be detected with largest probability. Inspired by the recent approaches
of \citet{JitXuSzaFukGre2017,SutTunStrDeRam2016,GreSejStrBalPon2012},
we propose optimizing the test locations in $V$ by maximizing the
asymptotic test power of the test statistic $\tph$. The test power
is defined as the probability of rejecting $H_{0}$ when it is false.
We start by giving the expression for the asymptotic test power of
$\tph$ in Corollary \ref{prop:fscd_power}. For brevity, we write
$\tpv$ for $\tpv(r)$.
\begin{cor}
\label{prop:fscd_power}Assume that $H_{1}$ holds. Given a set $V$
of test locations, and a rejection threshold $\gamma\in\mathbb{R}$,
the test power of the FSCD test is $P\left(\tpvh>\gamma\right)\approx\Phi\left(\sqrt{n}\frac{\tpv}{\sigma_{V}}-\frac{\gamma}{\sqrt{n}\sigma_{V}}\right)$
for sufficiently large $n$, where $\Phi$ is the CDF of the standard
normal distribution, and $\sigma_{V}=\sqrt{4\mathbb{V}[\mathbb{E}_{\mathbf{xy}}[\hbpv((\mathbf{x},\mathbf{y}),(\mathbf{x}',\mathbf{y}'))]]}$
is the standard deviation of the distribution of $\tpvh$ under $H_{1}$.
\end{cor}

The result directly follows from the fact that $\tpvh$ is asymptotically
normally distributed (see Proposition \ref{prop:fscd_asymptotics}
in the appendix). Following the same line of reasoning as in \citet{JitXuSzaFukGre2017,SutTunStrDeRam2016},
for large $n$, the power expression is dominated by $\tpv/\sigma_{V}$,
which is called the \emph{power criterion} \citep{JitXuSzaFukGre2017}.
Assume that $n$ is sufficiently large. It follows that finding the
test locations $V$ which maximize the test power amounts to finding
$V^{*}=\arg\max_{V}P\left(\tpvh>\gamma\right)\approx\arg\max_{V}\tpv/\sigma_{V}$.
We also use the same objective function to tune the two kernels $k$
and $l$.

To optimize, we split the data into two independent sets: training
and test sets. We then optimize this ratio with its consistent estimator
$\tpvh/\hat{\sigma}_{V}$ estimated from the training set. The hypothesis
test is performed on the test set using the optimized parameters.
Indeed, this data splitting scheme has also been used in several modern
statistical tests \citep{JitSzaChwGre2016,SutTunStrDeRam2016,JitKanSanHaySch2018,SceVar2019}.
There are two reasons for doing so: firstly, conducting a test on
an independent test set avoids overfitting to the training set ---
the false rejection rate of $H_{0}$ may be higher than the specified
significance level $\alpha$ otherwise; secondly, for the statistic
to be a U-statistic, its U-statistic kernel (i.e., $\hbpv$) must
be independent of the samples used to estimate the summands. In Section
\ref{sec:experiments}, we shall see that finding $V$ in this way
leads to a higher test power when the difference between $p$ and
$r$ is localized.

\section{EXPERIMENTS}

\label{sec:experiments}In this section, we empirically investigate
the two proposed tests.\footnote{Code is available at \url{https://github.com/wittawatj/kernel-cgof}.}

\textbf{1. Illustration of the FSCD power criterion} Our first task
is to illustrate that the power criterion of the proposed FSCD test
reveals where $p$ and $r$ differ in the domain of the conditioning
variable ($\mathbf{x}$). We consider a simple univariate problem
where the model is $p(y|x):=\mathcal{N}\left(x/2,1\right)$, the data
generating distribution is $r(y|x):=\mathcal{N}(x,1)$, and $r_{x}(x)=\mathcal{N}(0,1)$.
We use Gaussian kernels for both $k$ and $l$. The power criterion
function is shown in Figure \ref{fig:toy_illus}. More examples can
be found in Section \ref{sec:toy_powcri} (appendix).

\begin{figure}
\centering{}\includegraphics[width=0.85\columnwidth]{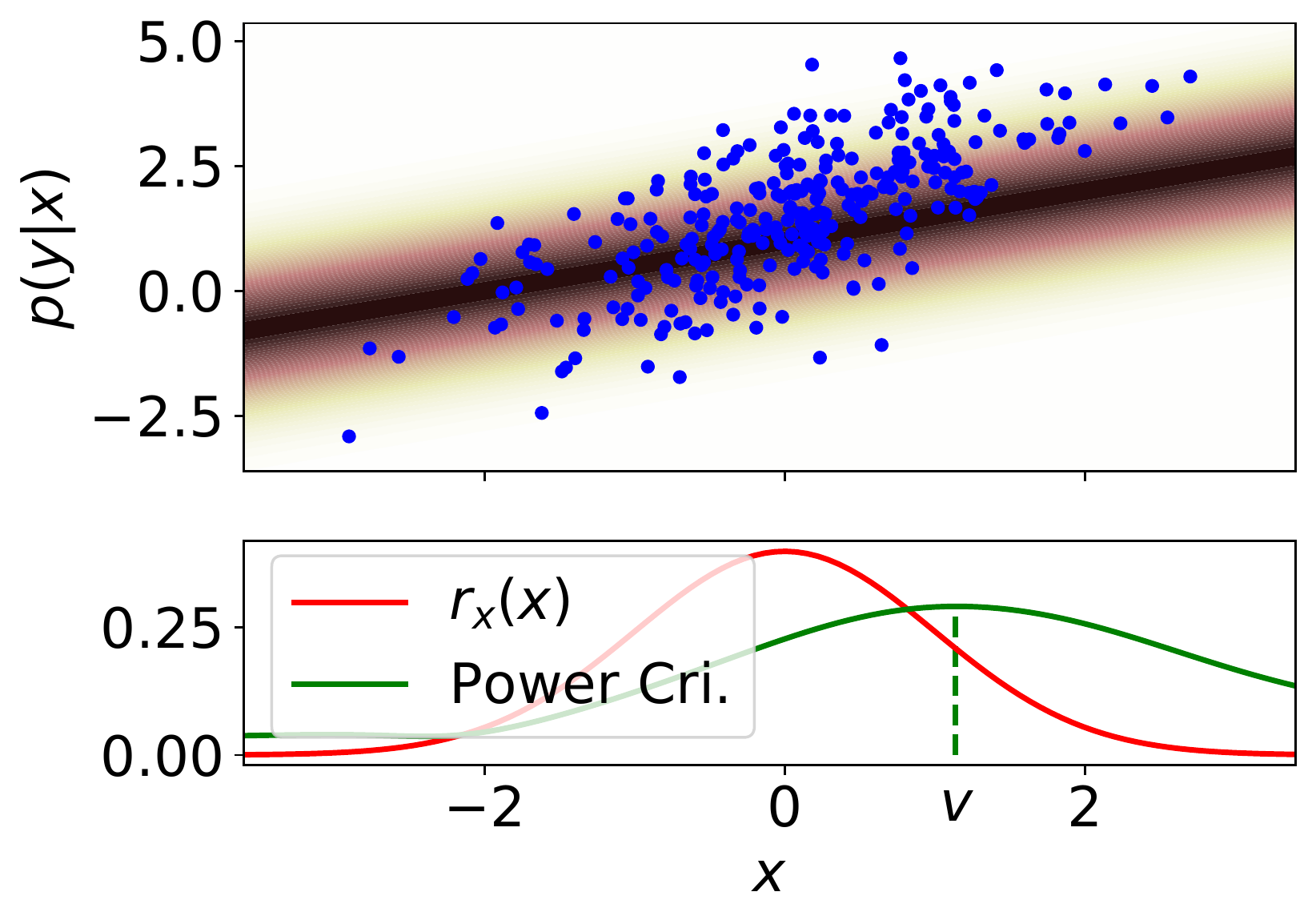}
\caption{\label{fig:toy_illus}The power criterion of FSCD as a function of
$x$ is high where the difference between $p(y|x)$ and $r(y|x)$
can be best detected.}
\vspace{-4mm}
\end{figure}

\textbf{2. Test power}
\begin{figure*}[t]
\centering
\vspace{-3mm}
\subfloat[Linear Gaussian Model ($H_0$)\label{fig:rej_lgm}]{
\includegraphics[width=0.265\linewidth]{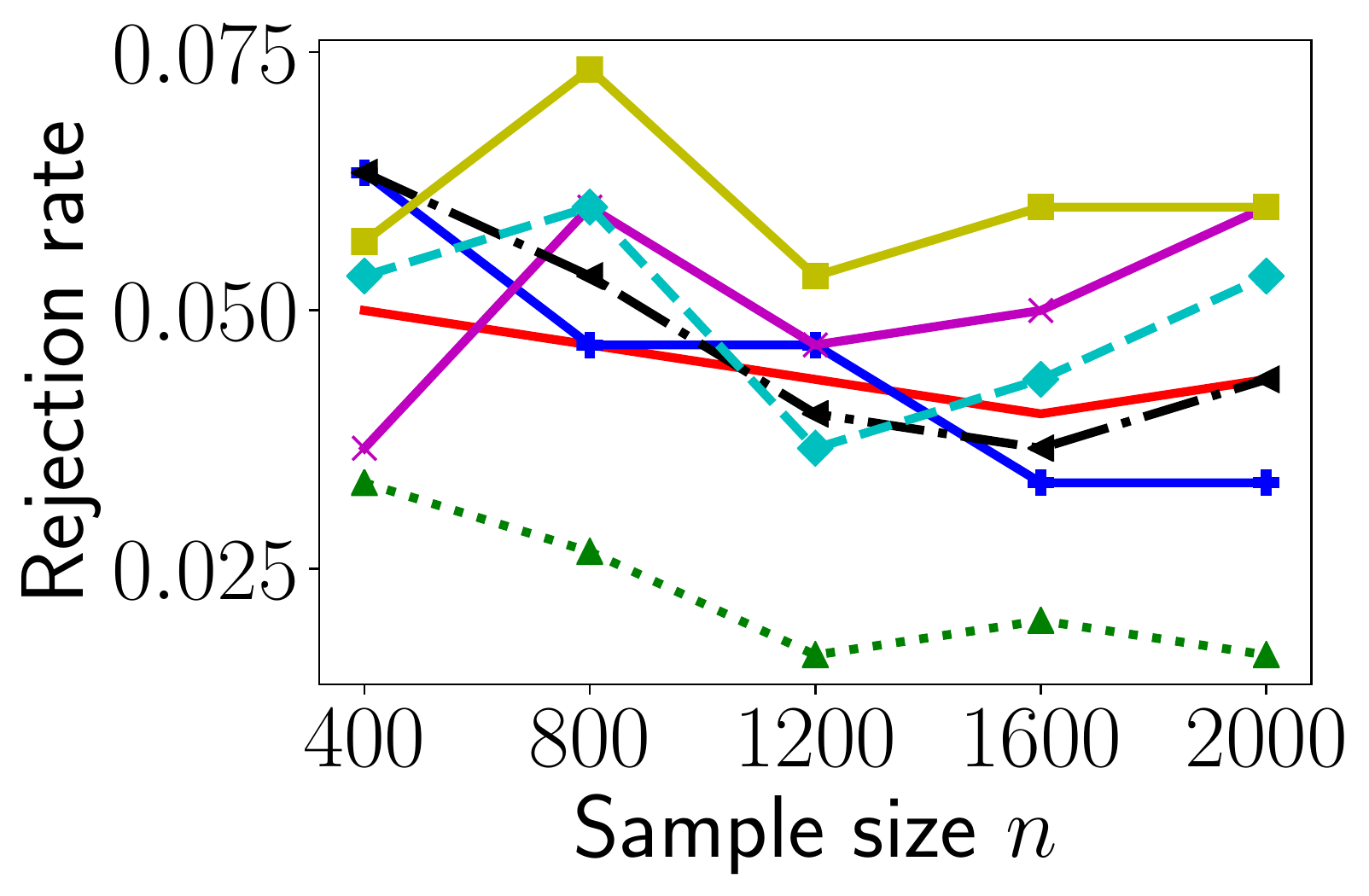}
}
\subfloat[Hetero. Gaussian Model ($H_1$)\label{fig:rej_hgm}]{
\includegraphics[width=0.265\linewidth]{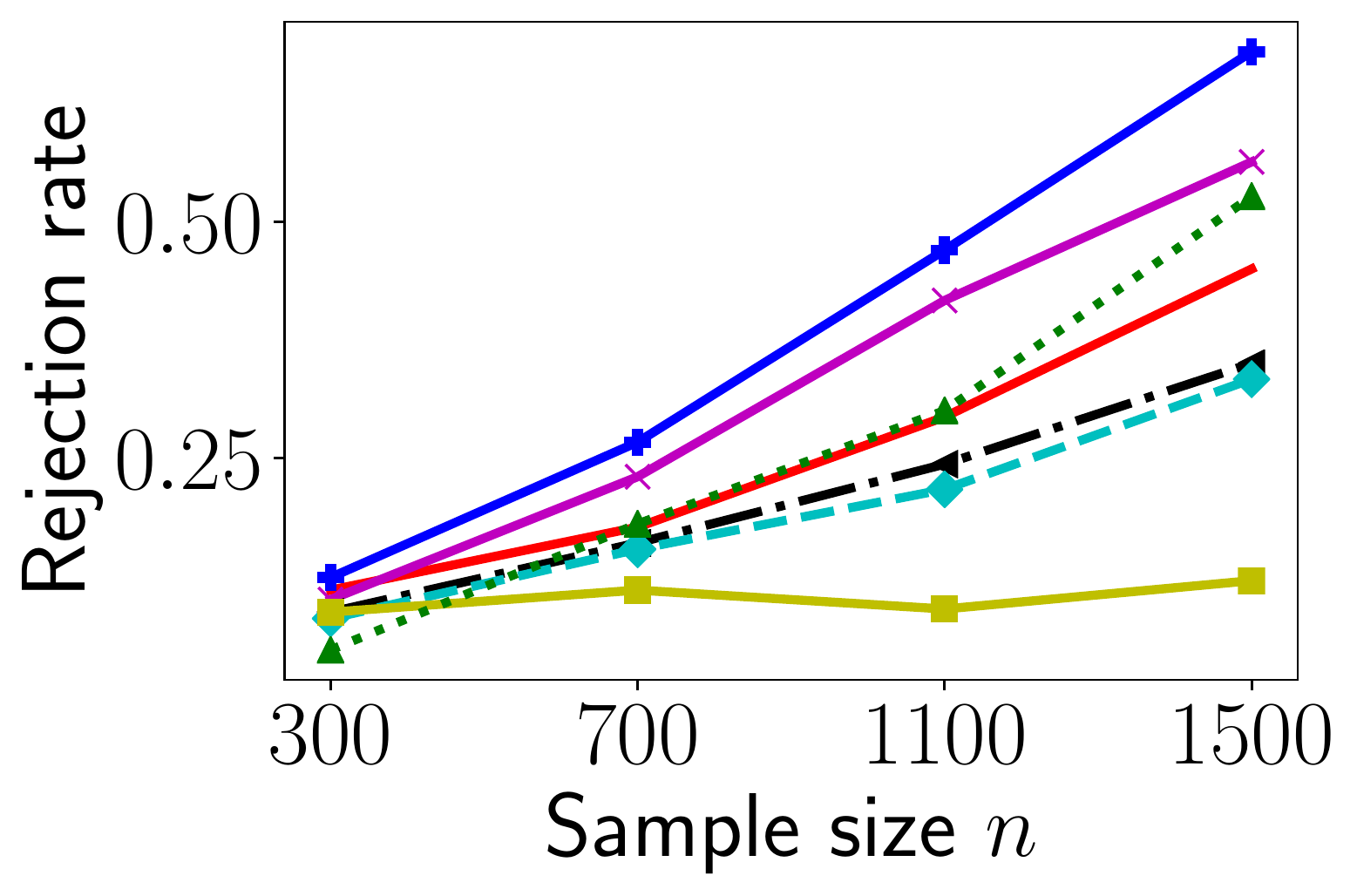}
}
\subfloat[Quadratic Gaussian Model ($H_1$)\label{fig:rej_qgm}]{
\includegraphics[width=0.26\linewidth]{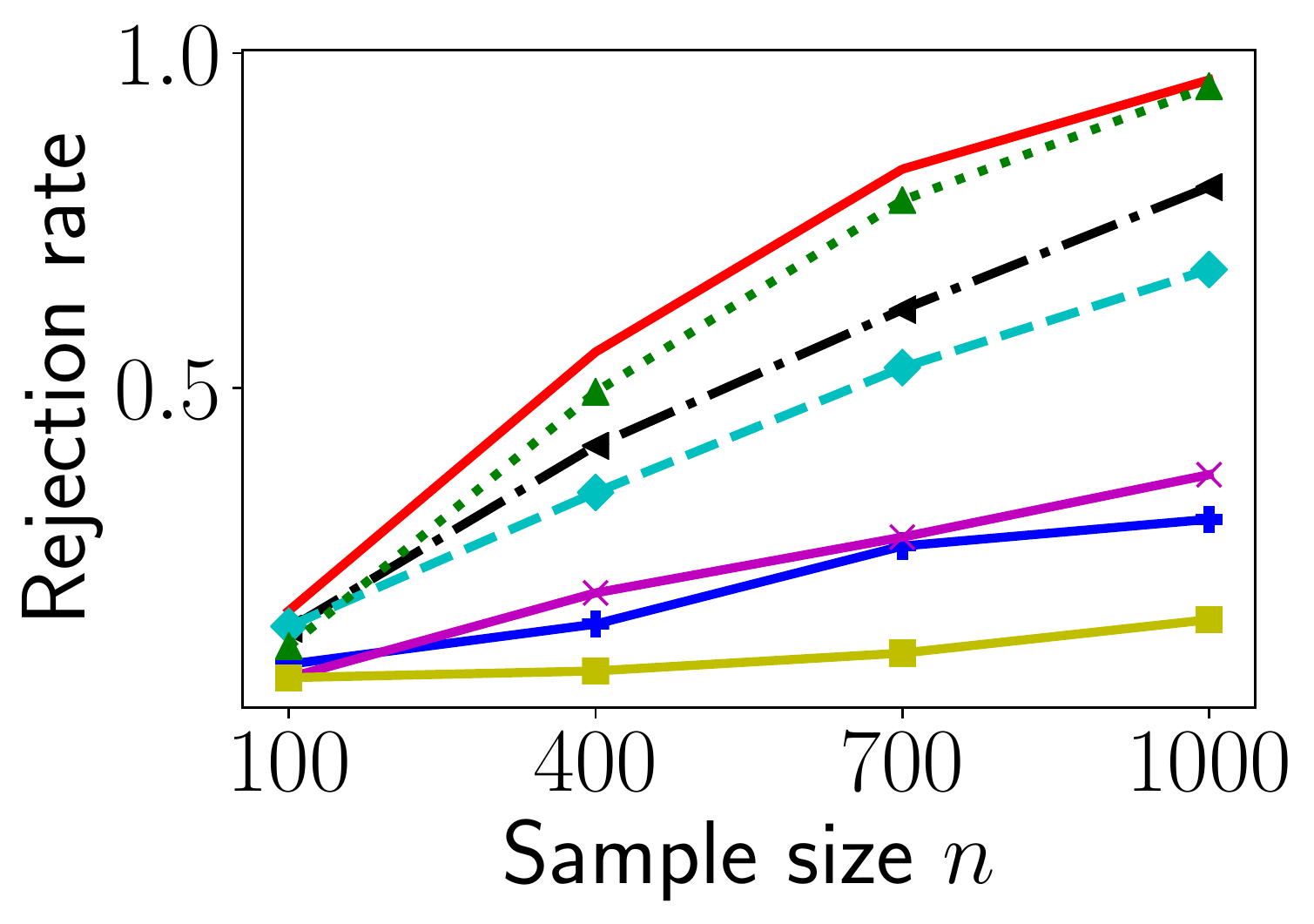}
\includegraphics[width=0.17\linewidth]{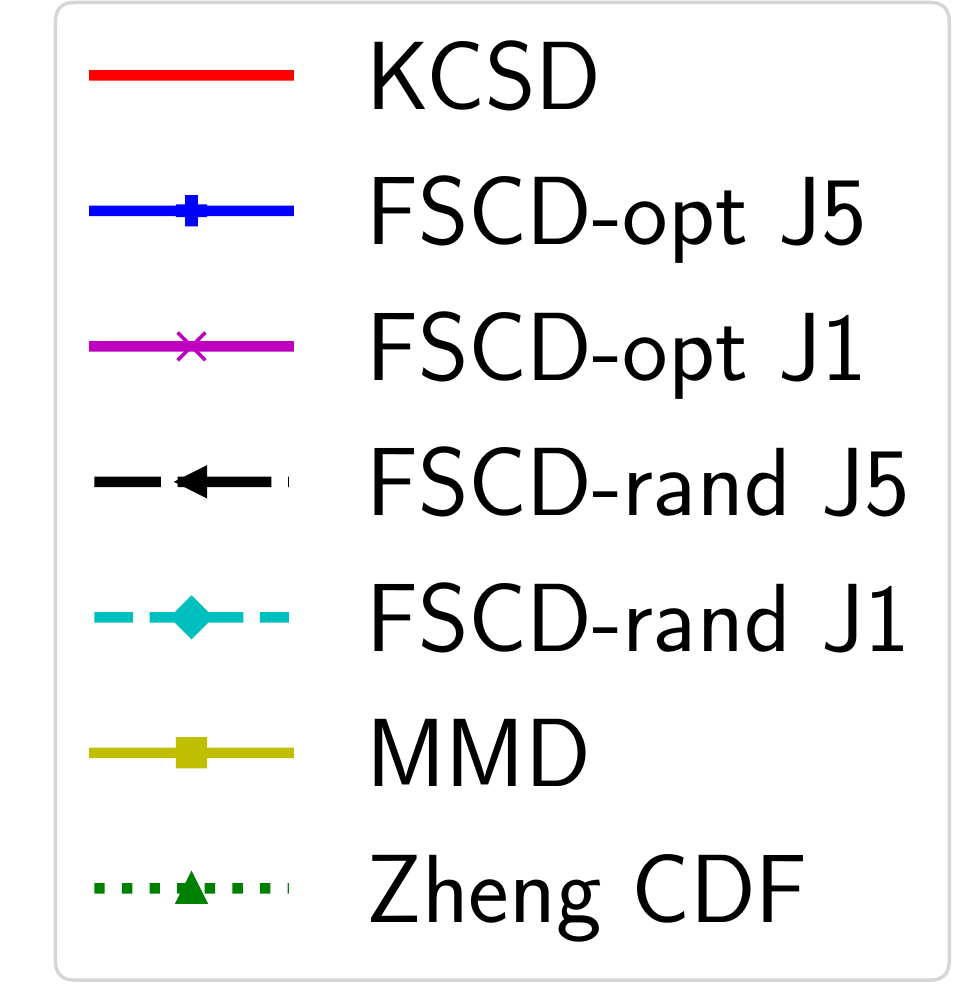} 
}\caption{Rejection rates of of the five tests with significance level $\alpha=0.05$.
\textbf{(a)}: $H_{0}$ is true. All test have false rejection rates
no larger than $\alpha$ (up to sampling noise). \textbf{(b)}: $H_{1}$
is true. FSCD-opt is good for detecting local difference. \textbf{(c)}:
KCSD is good for detecting global difference. \label{fig:toy_rej_rates}
\vspace{-2mm}}
\end{figure*}
We investigate the test power of the following methods.

\underline{\texttt{{KCSD}}}: our proposed KCSD test using Gaussian
kernels $k(\mathbf{x},\mathbf{x}')=\exp\left(-\frac{\|\mathbf{x}-\mathbf{x}'\|^{2}}{2\sigma_{x}^{2}}\right)$
and $l(\mathbf{y},\mathbf{y}')=\exp\left(-\frac{\|\mathbf{y}-\mathbf{y}'\|^{2}}{2\sigma_{y}^{2}}\right)$
where the bandwidths are set with $\sigma_{x}:=\mathrm{median}\left(\left\{ \|\mathbf{x}_{i}-\mathbf{x}_{j}\|_{2}\right\} _{i,j=1}^{n}\right)$
and $\sigma_{y}:=\mathrm{median}\left(\left\{ \|\mathbf{y}_{i}-\mathbf{y}_{j}\|_{2}\right\} _{i,j=1}^{n}\right)$.
This median heuristic has been used to set the bandwidth in many existing
kernel-based tests \citep{GreBorRasSchSmo2012,BouBelBlaAntGre2015,LiuLeeJor2016,ChwStrGre2016}.

\underline{\texttt{FSCD}}: our proposed FSCD test using Gaussian
kernels for $k$ and $l$. There are two variations of the FSCD. In
\textsf{FSCD-rand}, the $J$ test locations are randomly drawn from
a Gaussian distribution fitted to the data with maximum likelihood.
In the second variant \textsf{FSCD-opt}, 30\% of the observed data
are used for optimizing the two bandwidths and the $J$ test locations
by maximizing the power criterion, and the rest 70\% of the data are
used for testing.  All parameters of FSCD-opt are optimized jointly
with Adam \citep{KinBa2014} with default parameters implemented in
Pytorch. We consider $J\in\{1,5\}$.

\underline{\texttt{MMD}}: the Maximum Mean Discrepancy (MMD) test
\citep{GreBorRasSchSmo2012}. The MMD test was originally created
for two-sample testing. Here, we adapt it to conditional goodness-of-fit
testing by splitting the data into two disjoint sets $\{(\mathbf{x}_{i}^{(1)},\mathbf{y}_{i}^{(1)})\}_{i=1}^{n/2}$
and $\{(\mathbf{x}_{i}^{(2)},\mathbf{y}_{i}^{(2)})\}_{i=1}^{n/2}\}$
of equal size $n/2$. We then sample $\mathbf{y}'_{i}\sim p(\cdot|\mathbf{x}_{i}^{(2)})$
for each $i$. The test is performed on the first set, and $\{(\mathbf{x}_{i}^{(2)},\mathbf{y}_{i}')\}_{i=1}^{n/2}$.
The data splitting is performed to guarantee the independence between
the two sets of samples, which is a requirement of the MMD test. We
use the product of Gaussian kernels with bandwidths chosen by the
median heuristic. This approach serves as a nonparametric baseline
where the conditional model $p$ may be sampled easily.

\underline{\texttt{Zheng}}: Zheng\textquoteright s test \citep{Zhe2012}
is a specification test for parametric families of conditional distributions.
It is based on an (average) squared difference between the empirical
and the model CDFs, which is estimated by a U-statistic combined with
a kernel density estimator. We found that Epanechnikov kernel suggested
in \citep{Zhe2012} resulted in a poor performance and therefore choose
the standard Gaussian density as the smoothing kernel. We use a heuristic
similar to \citep{Zhe2012} to choose the kernel width parameters
$h_{j}=\hat{s}_{_{j}}n^{-1/(12d_{x})}$, where $h_{j}$ is the bandwidth
for the $j$-th coordinate of the covariate $\mathcal{\mathbf{x}}$,
and $\hat{s}_{j}$ the standard deviation of the coordinate.  The
test requires \textit{the best fitting parameter }in order to determine
the fit of a given parametric family. Instead of a maximum likelihood
estimator as proposed by \citet{Zhe2012}, in our experiments, the
reference and the model distributions share the same parameter values,
as the model family is a singleton set in our setting. 

These methods are tested on the following problems:

\textbf{Linear Gaussian Model (LGM)}: In this problem, $(\mathbf{x},y)\in\mathbb{R}^{5}\times\mathbb{R}$
and we set $p(y|\mathbf{x})=\mathcal{N}\left(\sum_{i=1}^{5}ix_{i},1\right)$,
set $r:=p$ and $r_{x}(\mathbf{x})=\mathcal{N}(\mathbf{0},\mathbf{I})$.
$H_{0}$ is true.

\textbf{Heteroscedastic Gaussian Model (HGM)}: $(\mathbf{x},y)\in\mathbb{R}^{3}\times\mathbb{R}$
and $p(y|\mathbf{x})=\mathcal{N}\left(\sum_{i=1}^{3}x_{i},\sigma^{2}(\mathbf{x})\right)$
where $\sigma^{2}(\mathbf{x}):=1+10\exp\left(-\frac{\|\mathbf{x}-\mathbf{c}\|^{2}}{2\times0.8^{2}}\right)$
and $\mathbf{c}=\frac{2}{3}\mathbf{1}$. We set the observation model
to be $r(y|\mathbf{x})=\mathcal{N}\left(\sum_{i=1}^{3}x_{i},1\right)$
and set $r_{x}(\mathbf{x})=\mathcal{N}(\mathbf{0},\mathbf{I})$. In
this problem, the observations are drawn from $r$ given by a linear
Gaussian model with unit variance. The model $p$ is heteroscedastic
(i.e., the noise depends on $\mathbf{x}$) where the variance function
is created such that it is roughly 1 everywhere in the domain of $\mathbf{x}$,
except in the region near $\mathbf{c}$. This problem is challenging
since the difference is local in $\mathcal{X}$. 

\textbf{Quadratic Gaussian Model (QGM)}: $(x,y)\in\mathbb{R}\times\mathbb{R}$
and we define $p(y|x)=\mathcal{N}\left(x+1,1\right)$, $r\left(y|x\right)=\mathcal{N}\left(0.1x^{2}+x+1,1\right)$,
and $r_{x}(x)=\mathrm{Uniform}\left(-2,2\right)$. Here, the conditional
mean of the true distribution $r$ is given by a quadratic function,
whereas the model $p$ is linear. This simulates a typical scenario
where the model is too simplistic to model the data. Note that the
quadratic term carries a small weight of 0.1, making the difference
between $p$ and $r$ challenging to detect. In this case, $H_{1}$
is true.

We report the rejection rates of these tests on all the three problems
in Figure \ref{fig:toy_rej_rates}, where we conduct 300 independent
trials for each experiment with the significance level set to $\alpha=0.05$.
In Figure \ref{fig:rej_lgm}, we observe that all the tests correctly
have their false rejection rates no larger than $\alpha=0.05$ (up
to sampling noise) since $H_{0}$ is true. In the HGM problem (Figure
\ref{fig:rej_hgm}) where the difference between $p$ and $r$ is
local in the domain $\mathcal{X}$, we observe the optimized test
locations of FSCD-opt are effective in identifying where to pinpoint
to difference in $\mathcal{X}$. This can be seen by noting that the
performance of FSCD-rand (random test locations) is significantly
lower than FSCD-opt, since the test locations are randomized, and
may be far from $\mathbf{c}$ which specifies the neighborhood that
reveals the difference (see the specification of the HGM problem).
While FSCD-opt has less test data since 30\% of the data is spent
on parameter tuning, the gain in the test power from having optimized
test locations in the right region outweighs the small reduction of
the test sample size.

In the QGM problem (Figure \ref{fig:rej_qgm}), while the quadratic
term in $r$ carries a small weight, as the sample size increases,
all the power of all the tests increases as expected. We observe
that the KCSD has higher performance than all variants of the FSCD
in this case. This is because the difference between $p$ and $r$
is spatially diffuse in a manner that a pointwise evaluation of $\mathbf{v}\mapsto\|G_{p,r}(\mathbf{v})\|_{\mathcal{F}_{l}^{d_{y}}}^{2}$
(recall the FSCD statistic in (\ref{eq:fscd_pop})) is small everywhere
in $\mathcal{X}=(-2,2)$. Thus, evaluating $G_{p,r}$ is less effective
in this problem. In the case where the difference is spatially diffuse,
it is more appropriate to take the norm of $G_{p,r}$, which explains
the superior performance of the KCSD. We also note that in constrast
to the HGM problem, in this case, FSCD-rand has higher performance
than FSCD-opt because there is no particular region in $\mathcal{X}$
that gives higher signal than other. As a result, optimizing for test
locations is less effective, and the test power drops because of smaller
test sample size. Finally, we observe that in both HGM and QGM problems,
the MMD has lower test power than other approaches due to the loss
of information from representing a model $p$ with samples. Zheng's
test performs well in the QGM problem. Its statistic given by the
expected squared difference between the empirical and the model CDFs
can be seen as capturing global differences. However, its use of kernel
density estimation may suffer when the data dimension is high as hinted
in Figure \ref{fig:rej_lgm} where it is overly conservative. 

\begin{figure}[t]
\centering{}\vspace{-3mm}\subfloat[Low power criterion \protect \\
(relatively good fit)\label{fig:nyc_low_pc}]{\includegraphics[width=0.47\columnwidth]{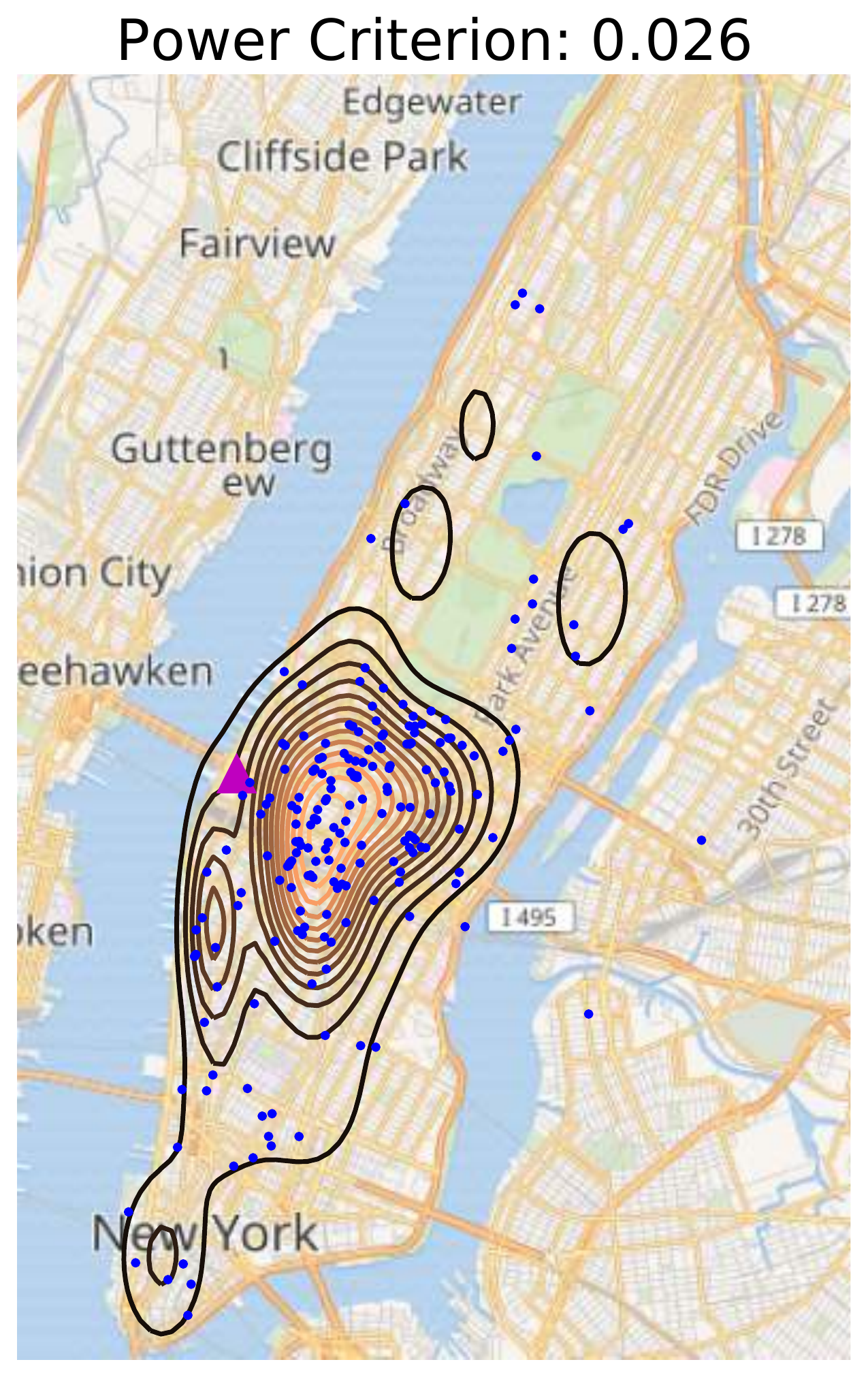}

\,\,\,}\subfloat[High power criterion \protect \\
(relatively poor fit)\label{fig:nyc_high_pc}]{\includegraphics[width=0.47\columnwidth]{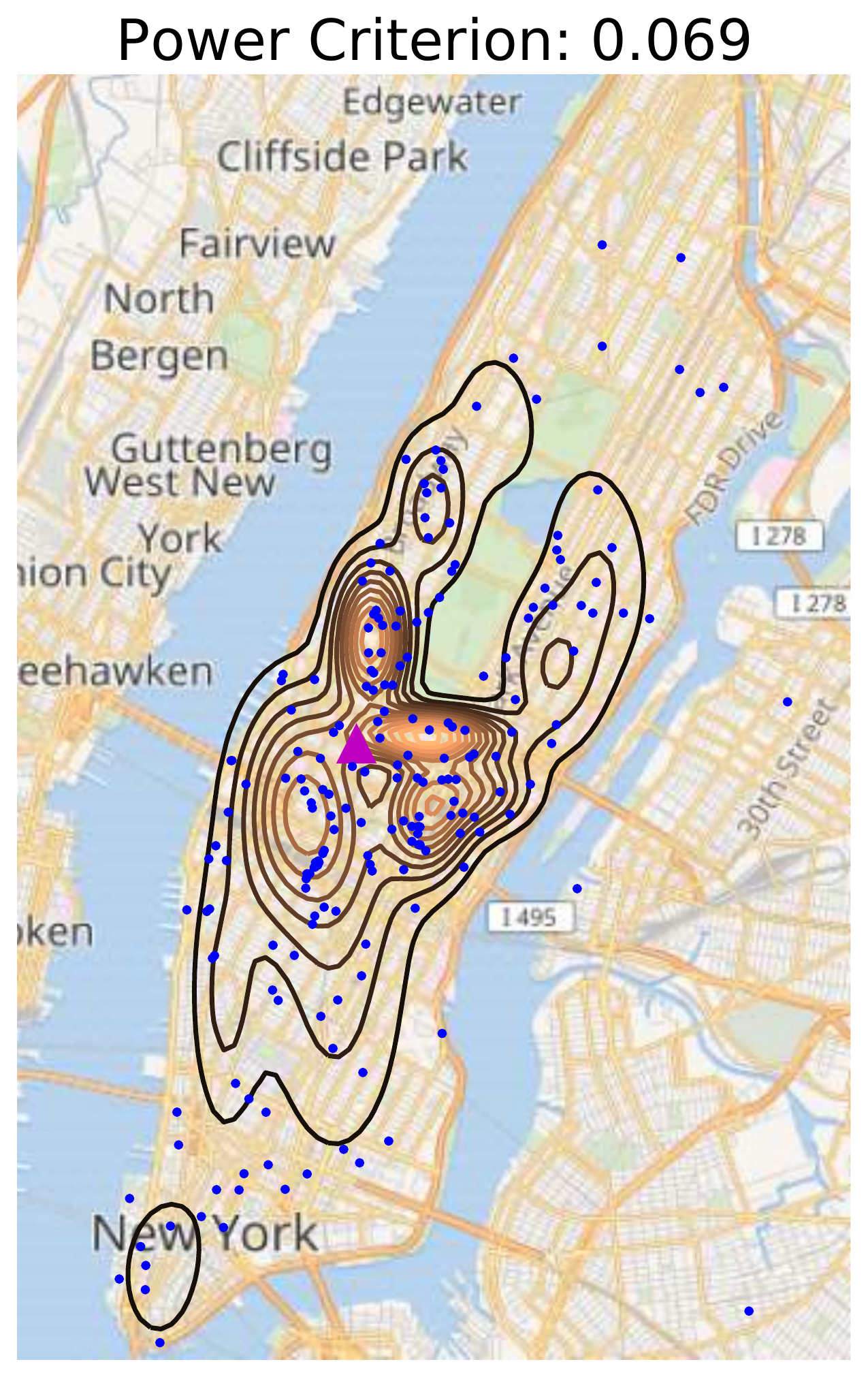}

}\caption{Mixture Density Network $p(\mathbf{y}|\mathbf{x})$ (black contour)
trained on five million records in the NYC taxi dataset. Here, $\mathbf{y}$
is the drop-off location and $\mathbf{x}$ is the pick-up location.
Blue points indicate real drop-off locations conditioned on the pick-up
location at $\blacktriangle$ (shown in purple). The FSCD power criterion
is evaluated at $J=1$ test location set to be at $\blacktriangle$.
Since the model $p(\mathbf{y}|\mathbf{x}=\blacktriangle)$ fits less
well in Figure \ref{fig:nyc_high_pc}, the power criterion is larger
than in Figure \ref{fig:nyc_low_pc}. \label{fig:nyc_power_criterion}
\vspace{-4mm}}
\end{figure}

\textbf{3. Informative power criterion} In our final experiment, we
show with real data that the power criterion of the FSCD, as a function
of $\mathbf{v}\in\mathcal{X}$ is a dimensionless quantity that roughly
coincides with the degree of mismatch between $p(\mathbf{y}|\mathbf{v})$
and the data. We train a Mixture Density Network (MDN, \citet[Section 5.6]{Bis2006})
on the New York City (NYC) taxi dataset. The dataset contains millions
of trip records that include pick-up locations, drop-off locations,
time, etc. The MDN models the conditional probability of the drop-off
location\textbf{ $\mathbf{y}$ }given a pick-up location $\mathbf{x}$,
expressed as a latitude/longitude coordinate (i.e., $\mathcal{X},\mathcal{Y}\subset\mathbb{R}^{2}$).
We train the model on five million trip records of yellow cabs from
January 2015 using 20 Gaussian components, and a ReLU-based architecture
for the mean, mixing proportion, and variance functions. For simplicity,
only trips with pick-up and drop-off locations within or close to
Manhattan are used. 

We use Gaussian kernels for both $k$ and $l$ with their bandwidths
chosen by the median heuristic, and separately compute the power criterion
of the FSCD test at two manually chosen test locations, using a held-out
data of size 12000. The results are shown in Figure \ref{fig:nyc_power_criterion}
where blue points indicate observed drop-off locations conditioned
on the pick-up location denoted by $\blacktriangle$. We consider
conditioning separately on two pick-up locations $\blacktriangle_{1}$
and $\blacktriangle_{2}$, shown in Figure \ref{fig:nyc_low_pc} and
Figure \ref{fig:nyc_high_pc}, respectively.

In Figure \ref{fig:nyc_low_pc}, $p(\mathbf{y}|\mathbf{x}=\blacktriangle_{1})$
fits relatively well to the data compared to $p(\mathbf{y}|\mathbf{x}=\blacktriangle_{2})$
shown in \ref{fig:nyc_high_pc}. In Figure \ref{fig:nyc_high_pc},
the observed data (blue) do not respect the multimodality suggested
by the model. As a result, the power criterion evaluated at $\blacktriangle_{2}$
is higher, indicating a poorer fit at $\blacktriangle_{2}$. This
suggests that the power criterion function of the FSCD gives an interpretable
indication for where the conditional model does not fit well. More
details on the MDN and more results can be found in Section \ref{sec:nyc_taxi_experiment}
(appendix).

\section{CONCLUSION}

We have proposed two novel conditional goodness-of-fit tests: the
Kernel Conditional Stein Discrepancy (KCSD), and the Finite Set Conditional
Discrepancy (FSCD). We prove that the population statistics of the
two test define a proper divergence measure between two conditional
density functions. There are several possible future directions.
Both KCSD and FSCD can be extended to handle a discrete domain $\mathcal{Y}$
by considering a Stein operator defined in terms of forward and backward
differences as in \Citet{YanLiuRaoNev2018}. Further, our two tests
can be sped up to have a runtime complexity linear in the sample size
(instead of quadratic in the current version) by considering random
Fourier features as in \citet{HugMac2018}. The two tests can also
be extended to compare the relative fit of two competing models as
in \citet{JitKanSanHaySch2018,BouBelBlaAntGre2015}. We leave these
research directions for future work.

\subsubsection*{Acknowledgment}

We thank Patsorn Sangkloy for helping us with the experiment on the
NYC taxi dataset. HK thanks the Gatsby Charitable Foundation for the
financial support


{\small \bibliographystyle{myplainnat}
\bibliography{kcgof_ref}

\begin{thebibliography}{41}
\providecommand{\natexlab}[1]{#1}
\providecommand{\url}[1]{\texttt{#1}}
\expandafter\ifx\csname urlstyle\endcsname\relax
  \providecommand{\doi}[1]{doi: #1}\else
  \providecommand{\doi}{doi: \begingroup \urlstyle{rm}\Url}\fi

\bibitem[Andrews(1997)]{And1997}
D.~W.~K. Andrews.
\newblock A conditional {Kolmogorov} test.
\newblock \emph{Econometrica}, 65\penalty0 (5):\penalty0 1097--1128, 1997.

\bibitem[Arcones and Gine(1992)]{ArcGin1992}
M.~A. Arcones and E.~Gine.
\newblock On the bootstrap of {U} and {V} statistics.
\newblock \emph{The Annals of Statistics}, pages 655--674, 1992.

\bibitem[Berlinet and Thomas-Agnan(2011)]{BerTho2011}
A.~Berlinet and C.~Thomas-Agnan.
\newblock \emph{Reproducing kernel Hilbert spaces in probability and
  statistics}.
\newblock Springer Science \& Business Media, 2011.

\bibitem[Bierens(1982)]{Bie1982}
H.~J. Bierens.
\newblock Consistent model specification tests.
\newblock \emph{Journal of Econometrics}, 20\penalty0 (1):\penalty0 105 -- 134,
  1982.
\newblock ISSN 0304-4076.

\bibitem[Bierens(1990)]{Bie1990}
H.~J. Bierens.
\newblock A consistent conditional moment test of functional form.
\newblock \emph{Econometrica: Journal of the Econometric Society}, pages
  1443--1458, 1990.

\bibitem[Bierens and Ploberger(1997)]{BiePlo1997}
H.~J. Bierens and W.~Ploberger.
\newblock Asymptotic theory of integrated conditional moment tests.
\newblock \emph{Econometrica: Journal of the Econometric Society}, pages
  1129--1151, 1997.

\bibitem[Bishop(2006)]{Bis2006}
C.~M. Bishop.
\newblock \emph{Pattern recognition and machine learning}.
\newblock Springer, 2006.

\bibitem[Bounliphone et~al.(2015)Bounliphone, Belilovsky, Blaschko, Antonoglou,
  and Gretton]{BouBelBlaAntGre2015}
W.~Bounliphone, E.~Belilovsky, M.~B. Blaschko, I.~Antonoglou, and A.~Gretton.
\newblock A test of relative similarity for model selection in generative
  models.
\newblock In \emph{ICLR}, 2015.

\bibitem[{Carmeli} et~al.(2008){Carmeli}, {De Vito}, {Toigo}, and
  {Umanit{\`a}}]{CarDeToiUma2008}
C.~{Carmeli}, E.~{De Vito}, A.~{Toigo}, and V.~{Umanit{\`a}}.
\newblock Vector valued reproducing kernel {Hilbert} spaces and universality.
\newblock \emph{arXiv e-prints}, Jul 2008.

\bibitem[Carmeli et~al.(2006)Carmeli, De~Vito, and Toigo]{CarDeToi2006}
C.~Carmeli, E.~De~Vito, and A.~Toigo.
\newblock Vector valued reproducing kernel {Hilbert} spaces of integrable
  functions and {Mercer} theorem.
\newblock \emph{Analysis and Applications}, 4\penalty0 (04):\penalty0 377--408,
  2006.

\bibitem[Chwialkowski et~al.(2015)Chwialkowski, Ramdas, Sejdinovic, and
  Gretton]{ChwRamSejGre2015}
K.~Chwialkowski, A.~Ramdas, D.~Sejdinovic, and A.~Gretton.
\newblock Fast two-sample testing with analytic representations of probability
  measures.
\newblock In \emph{Advances in Neural Information Processing Systems}, pages
  1981--1989, 2015.

\bibitem[Chwialkowski et~al.(2016)Chwialkowski, Strathmann, and
  Gretton]{ChwStrGre2016}
K.~Chwialkowski, H.~Strathmann, and A.~Gretton.
\newblock A kernel test of goodness of fit.
\newblock In \emph{ICML}, pages 2606--2615, 2016.

\bibitem[Dutordoir et~al.(2018)Dutordoir, Salimbeni, Hensman, and
  Deisenroth]{DutSalHenDei2018}
V.~Dutordoir, H.~Salimbeni, J.~Hensman, and M.~Deisenroth.
\newblock Gaussian process conditional density estimation.
\newblock In \emph{NeurIPS}, pages 2385--2395, 2018.

\bibitem[Gorham and Mackey(2017)]{GorMac2017}
J.~Gorham and L.~Mackey.
\newblock Measuring sample quality with kernels.
\newblock In \emph{ICML}, pages 1292--1301, 2017.

\bibitem[Gretton et~al.(2008)Gretton, Fukumizu, Teo, Song, Sch{\"o}lkopf, and
  Smola]{GreFukTeoSonSch2008}
A.~Gretton, K.~Fukumizu, C.~H. Teo, L.~Song, B.~Sch{\"o}lkopf, and A.~J. Smola.
\newblock A kernel statistical test of independence.
\newblock In \emph{NeurIPS}, pages 585--592, 2008.

\bibitem[Gretton et~al.(2009)Gretton, Fukumizu, Harchaoui, and
  Sriperumbudur]{GreFukHarSri2009}
A.~Gretton, K.~Fukumizu, Z.~Harchaoui, and B.~K. Sriperumbudur.
\newblock A fast, consistent kernel two-sample test.
\newblock In \emph{NeurIPS}, pages 673--681. 2009.

\bibitem[Gretton et~al.(2012{\natexlab{a}})Gretton, Borgwardt, Rasch,
  Sch{\"o}lkopf, and Smola]{GreBorRasSchSmo2012}
A.~Gretton, K.~M. Borgwardt, M.~J. Rasch, B.~Sch{\"o}lkopf, and A.~Smola.
\newblock A kernel two-sample test.
\newblock \emph{Journal of Machine Learning Research}, 13:\penalty0 723--773,
  2012{\natexlab{a}}.

\bibitem[Gretton et~al.(2012{\natexlab{b}})Gretton, Sejdinovic, Strathmann,
  Balakrishnan, Pontil, Fukumizu, and Sriperumbudur]{GreSejStrBalPon2012}
A.~Gretton, D.~Sejdinovic, H.~Strathmann, S.~Balakrishnan, M.~Pontil,
  K.~Fukumizu, and B.~K. Sriperumbudur.
\newblock Optimal kernel choice for large-scale two-sample tests.
\newblock In \emph{NeurIPS}, pages 1205--1213, 2012{\natexlab{b}}.

\bibitem[Huggins and Mackey(2018)]{HugMac2018}
J.~Huggins and L.~Mackey.
\newblock Random feature {Stein} discrepancies.
\newblock In \emph{NeurIPS}, pages 1899--1909, 2018.

\bibitem[Huskova and Janssen(1993)]{HusJan1993}
M.~Huskova and P.~Janssen.
\newblock Consistency of the generalized bootstrap for degenerate
  $u$-statistics.
\newblock \emph{Ann. Statist.}, 21\penalty0 (4):\penalty0 1811--1823, 12 1993.

\bibitem[Jitkrittum et~al.(2016)Jitkrittum, Szab\'{o}, Chwialkowski, and
  Gretton]{JitSzaChwGre2016}
W.~Jitkrittum, Z.~Szab\'{o}, K.~P. Chwialkowski, and A.~Gretton.
\newblock Interpretable distribution features with maximum testing power.
\newblock In \emph{NeurIPS}, pages 181--189. 2016.

\bibitem[Jitkrittum et~al.(2017{\natexlab{a}})Jitkrittum, Szab{\'o}, and
  Gretton]{JitSzaGre2017}
W.~Jitkrittum, Z.~Szab{\'o}, and A.~Gretton.
\newblock An adaptive test of independence with analytic kernel embeddings.
\newblock In \emph{ICML}. 2017{\natexlab{a}}.

\bibitem[Jitkrittum et~al.(2017{\natexlab{b}})Jitkrittum, Xu, Szabo, Fukumizu,
  and Gretton]{JitXuSzaFukGre2017}
W.~Jitkrittum, W.~Xu, Z.~Szabo, K.~Fukumizu, and A.~Gretton.
\newblock A linear-time kernel goodness-of-fit test.
\newblock In \emph{NeurIPS}, 2017{\natexlab{b}}.

\bibitem[Jitkrittum et~al.(2018)Jitkrittum, Kanagawa, Sangkloy, Hays,
  Sch{\"o}lkopf, and Gretton]{JitKanSanHaySch2018}
W.~Jitkrittum, H.~Kanagawa, P.~Sangkloy, J.~Hays, B.~Sch{\"o}lkopf, and
  A.~Gretton.
\newblock Informative features for model comparison.
\newblock In \emph{NeurIPS}, pages 808--819, 2018.

\bibitem[Kingma and Ba(2014)]{KinBa2014}
D.~P. Kingma and J.~Ba.
\newblock Adam: A method for stochastic optimization.
\newblock \emph{arXiv preprint arXiv:1412.6980}, 2014.

\bibitem[Liu et~al.(2016)Liu, Lee, and Jordan]{LiuLeeJor2016}
Q.~Liu, J.~Lee, and M.~Jordan.
\newblock A kernelized {Stein} discrepancy for goodness-of-fit tests.
\newblock In \emph{ICML}, pages 276--284, 2016.

\bibitem[{Mityagin}(2015)]{Mit2015}
B.~{Mityagin}.
\newblock The zero set of a real analytic function.
\newblock \emph{arXiv e-prints}, art. arXiv:1512.07276, Dec 2015.

\bibitem[Moreira(2003)]{Mor2003}
M.~J. Moreira.
\newblock A conditional likelihood ratio test for structural models.
\newblock \emph{Econometrica}, 71\penalty0 (4):\penalty0 1027--1048, 2003.

\bibitem[Oates et~al.(2017)Oates, Girolami, and Chopin]{OatGirCho2017}
C.~J. Oates, M.~Girolami, and N.~Chopin.
\newblock Control functionals for {Monte} {Carlo} integration.
\newblock \emph{Journal of the Royal Statistical Society: Series B (Statistical
  Methodology)}, 79\penalty0 (3):\penalty0 695--718, 2017.
\newblock \doi{10.1111/rssb.12185}.

\bibitem[Scetbon and Varoquaux(2019)]{SceVar2019}
M.~Scetbon and G.~Varoquaux.
\newblock Comparing distributions: L1 geometry improves kernel two-sample
  testing.
\newblock In \emph{NeurIPS}, pages 12306--12316. 2019.

\bibitem[Serfling(2009)]{Ser2009}
R.~J. Serfling.
\newblock \emph{Approximation {Theorems} of {Mathematical} {Statistics}}.
\newblock John Wiley \& Sons, 2009.

\bibitem[Sriperumbudur et~al.(2011)Sriperumbudur, Fukumizu, and
  Lanckriet]{SriFukLan2011}
B.~K. Sriperumbudur, K.~Fukumizu, and G.~R.~G. Lanckriet.
\newblock Universality, characteristic kernels and {RKHS} embedding of
  measures.
\newblock \emph{Journal of Machine Learning Research}, 12:\penalty0 2389--2410,
  2011.

\bibitem[Steinwart and Christmann(2008)]{SteChr2008}
I.~Steinwart and A.~Christmann.
\newblock \emph{Support vector machines}.
\newblock Springer Science \& Business Media, 2008.

\bibitem[Stute and Zhu(2002)]{StuZhu2002}
W.~Stute and L.-X. Zhu.
\newblock Model checks for generalized linear models.
\newblock \emph{Scandinavian Journal of Statistics}, 29\penalty0 (3):\penalty0
  535--545, 2002.
\newblock ISSN 03036898, 14679469.

\bibitem[Sutherland et~al.(2016)Sutherland, Tung, Strathmann, De, Ramdas,
  Smola, and Gretton]{SutTunStrDeRam2016}
D.~J. Sutherland, H.-Y. Tung, H.~Strathmann, S.~De, A.~Ramdas, A.~Smola, and
  A.~Gretton.
\newblock Generative models and model criticism via optimized maximum mean
  discrepancy.
\newblock In \emph{ICLR}. 2016.

\bibitem[Szab{\'o} and Sriperumbudur(2018)]{SzaSri2018}
Z.~Szab{\'o} and B.~K. Sriperumbudur.
\newblock Characteristic and universal tensor product kernels.
\newblock \emph{Journal of Machine Learning Research}, 18\penalty0
  (233):\penalty0 1--29, 2018.

\bibitem[Tripathi et~al.(2003)Tripathi, Kitamura, et~al.]{TriKitoth2003}
G.~Tripathi, Y.~Kitamura, et~al.
\newblock Testing conditional moment restrictions.
\newblock \emph{The Annals of Statistics}, 31\penalty0 (6):\penalty0
  2059--2095, 2003.

\bibitem[Uria et~al.(2016)Uria, C{\^o}t{\'e}, Gregor, Murray, and
  Larochelle]{UriCotGreMurLar2016}
B.~Uria, M.-A. C{\^o}t{\'e}, K.~Gregor, I.~Murray, and H.~Larochelle.
\newblock Neural autoregressive distribution estimation.
\newblock \emph{The Journal of Machine Learning Research}, 17\penalty0
  (1):\penalty0 7184--7220, 2016.

\bibitem[Yang et~al.(2018)Yang, Liu, Rao, and Neville]{YanLiuRaoNev2018}
J.~Yang, Q.~Liu, V.~Rao, and J.~Neville.
\newblock Goodness-of-fit testing for discrete distributions via {Stein}
  discrepancy.
\newblock In \emph{ICML}, pages 5561--5570, 2018.

\bibitem[Zheng(2000)]{Zhe2000}
J.~X. Zheng.
\newblock A consistent test of conditional parametric distributions.
\newblock \emph{Econometric Theory}, 16\penalty0 (5):\penalty0 667--691, 2000.

\bibitem[Zheng(2012)]{Zhe2012}
X.~Zheng.
\newblock Testing parametric conditional distributions using the nonparametric
  smoothing method.
\newblock \emph{Metrika}, 75\penalty0 (4):\penalty0 455--469, May 2012.

\end{thebibliography}
}

\clearpage
\newpage
\appendix
\onecolumn
\begin{center}
{\LARGE{}\ourtitle{}}{\LARGE\par}
\par\end{center}

\begin{center}
\textcolor{black}{\Large{}Supplementary}{\Large\par}
\par\end{center}

\section{PROOFS}

\label{sec:proofs}This section contains proofs of the theoretical
results we gave in the main text. We first give two known lemmas that
will be needed.
\begin{lem}[{\citealp[Theorem 2b, Section 4]{CarDeToiUma2008} (rephrased)}]
 \label{lem:carmeli_injective}Let $\mathcal{X}$ be a locally compact
second countable topological space, and $\mathcal{Z}$ be a complex
separable Hilbert space. Let $K\colon\mathcal{X}\times\mathcal{X}\to\mathcal{L}(\mathcal{Z})$
be a $C_{0}$ universal kernel associated with the vector-valued RKHS
$\mathcal{F}_{K}$, where $\mathcal{L}(\mathcal{Z})$ denotes the
Banach space of bounded operators from $\mathcal{Z}$ to $\mathcal{Z}$.
Let $P$ be a probability measure on $\mathcal{X}$. Then, the operator
$L_{P}\colon L^{2}(\mathcal{X},P;\mathcal{Z})\rightarrow\mathcal{F}_{K}$
given by $(L_{P}f)(\mathbf{t})=\int_{\mathcal{X}}K(\mathbf{t},\mathbf{x})f(\mathbf{x})\thinspace\mathrm{d}P(\mathbf{x})$
is injective, for all $f\in L^{2}(\mathcal{X},P;\mathcal{Z})$.
\end{lem}

\begin{lem}[{\citealt[Lemma 1]{ChwRamSejGre2015}}]
 \label{lem:rkhs_analytic}If $k\colon\mathbb{R}^{d_{x}}\times\mathbb{R}^{d_{x}}\to\mathbb{R}$
is a bounded, real analytic kernel (i.e., for any $\mathbf{v}\in\mathcal{X}$,
$\mathbf{x}\mapsto k(\mathbf{x},\mathbf{v})$ is a real analytic function),
then all functions in the RKHS defined by $k$ are real analytic.
\end{lem}

\subsection{PROOF OF THEOREM \ref{thm:pop_dp}}

\label{subsec:proof_popstatdiv}Recall the theorem: \popstatdiv*
\begin{proof}
We first rewrite the statistic as 
\begin{align*}
D_{p}^{2}(r) & =\big\|\mathbb{E}_{(\mathbf{x},\mathbf{y})\sim r_{xy}}K_{\mathbf{x}}\xi_{p_{|\mathbf{x}}}(\mathbf{y},\diamond)\big\|_{\mathcal{F}_{K}}^{2}\\
 & =\big\|\mathbb{E}_{\mathbf{x}\sim r_{x}}K_{\mathbf{x}}\mathbb{E}_{\mathbf{y}\sim r_{|\mathbf{x}}}\xi_{p_{|\mathbf{x}}}(\mathbf{y},\diamond)\big\|_{\mathcal{F}_{K}}^{2}\\
 & =\big\|\mathbb{E}_{\mathbf{x}\sim r_{x}}K_{\mathbf{x}}\mathbf{g}_{p,r}(\diamond|\mathbf{x})\big\|_{\mathcal{F}_{K}}^{2},
\end{align*}
where $\mathbf{g}_{p,r}(\mathbf{w}|\mathbf{x}):=\mathbb{E}_{\mathbf{y}\sim r_{|\mathbf{x}}}\xi_{p_{|\mathbf{x}}}(\mathbf{y},\mathbf{w})\in\mathbb{R}^{d_{y}}$
is the Stein witness function between $p_{|\mathbf{x}}$ and $r_{|\mathbf{x}}$,
and $\xi_{p_{|\mathbf{x}}}(\mathbf{y},\cdot):=l(\mathbf{y},\cdot)\nabla_{\mathbf{y}}\log p(\mathbf{y}|\mathbf{x})+\nabla_{\mathbf{y}}l(\mathbf{y},\cdot)\in\mathcal{F}_{l}^{d_{y}}$
for $r_{x}$-almost all $\mathbf{x}$ \citep{ChwStrGre2016,LiuLeeJor2016,JitXuSzaFukGre2017}.
By \citet[Theorem 2.2]{ChwStrGre2016}, for $r_{x}$-almost all $\mathbf{x}\in\mathcal{X}$,
the Kernel Stein Discrepancy (KSD) between the two probability density
functions $p_{|\mathbf{x}}$ and $r_{|\mathbf{x}}$ is 0 if and only
if they coincide. That is, given $\mathbf{x}\sim r_{x}$, $\mathrm{KSD}_{p_{|\mathbf{x}}}^{2}(r_{|\mathbf{x}})=0=\|\mathbf{g}_{p,r}(\diamond|\mathbf{x})\|_{\mathcal{F}_{l}^{d_{y}}}^{2}$
if and only if $p_{|\mathbf{x}}=r_{|\mathbf{x}}$. Thus, proving the
claim amounts to showing $\mathbf{g}_{p,r}(\diamond|\mathbf{x})=\mathbf{0}$
for $r_{x}$-almost all $\mathbf{x}$ if and only if $p\rxeq r$.
Since $\mathbf{g}_{p,r}\in L^{2}(\mathcal{X},r_{x};\mathcal{F}_{l}^{d_{y}})$
and $K$ is $C_{0}$-universal, Lemma \ref{lem:carmeli_injective}
(by setting $\mathcal{Z}=\mathcal{F}_{l}^{d_{y}}$) implies that the
map $\mathbf{g}_{p,r}\mapsto\mathbb{E}_{\mathbf{x}\sim r_{x}}K_{\mathbf{x}}\mathbf{g}_{p,r}(\diamond|\mathbf{x})$
is injective. As a result of the injectivity and the fact that $A\mathbf{0}=\mathbf{0}$
if $A$ is a linear operator, we have $\mathbb{E}_{\mathbf{x}\sim r_{x}}K_{\mathbf{x}}\mathbf{g}_{p,r}(\diamond|\mathbf{x})=\mathbf{0}$
if and only if $\mathbf{g}_{p,r}=\mathbf{0}$ or equivalently $\mathbf{g}_{p,r}(\diamond|\mathbf{x})=\mathbf{0}$
for all $r_{x}$-almost all $\mathbf{x}$.
\end{proof}

\subsection{PROOF OF PROPOSITION \ref{prop:kssd_pop_ustat}}

From (\ref{eq:kssd_pop}), we have
\begin{align*}
D_{p}(r) & =\big\|\mathbb{E}_{(\mathbf{x},\mathbf{y})\sim r_{xy}}K_{\mathbf{x}}\xi_{p_{|\mathbf{x}}}(\mathbf{y},\diamond)\big\|_{\mathcal{F}_{K}}^{2}\\
 & =\left\langle \mathbb{E}_{\mathbf{xy}}K_{\mathbf{x}}\xi_{p_{|\mathbf{x}}}(\mathbf{y},\diamond),\mathbb{E}_{\mathbf{x'y'}}K_{\mathbf{x'}}\xi_{p_{|\mathbf{x}'}}(\mathbf{y}',\diamond)\right\rangle _{\mathcal{F}_{K}}\\
 & \stackrel{(a)}{=}\mathbb{E}_{\mathbf{xy}}\mathbb{E}_{\mathbf{x'y'}}\left\langle K_{\mathbf{x}}\xi_{p_{|\mathbf{x}}}(\mathbf{y},\diamond),K_{\mathbf{x'}}\xi_{p_{|\mathbf{x}'}}(\mathbf{y}',\diamond)\right\rangle _{\mathcal{F}_{K}}\\
 & \stackrel{(b)}{=}\mathbb{E}_{\mathbf{xy}}\mathbb{E}_{\mathbf{x'y'}}\left\langle K_{\mathbf{x}'}^{*}K_{\mathbf{x}}\xi_{p_{|\mathbf{x}}}(\mathbf{y},\diamond),\xi_{p_{|\mathbf{x}'}}(\mathbf{y}',\diamond)\right\rangle _{\mathcal{F}_{l}^{d_{y}}}\\
 & =\mathbb{E}_{\mathbf{xy}}\mathbb{E}_{\mathbf{x'y'}}k(\mathbf{x},\mathbf{x}')h_{p}((\mathbf{x},\mathbf{y}),(\mathbf{x}',\mathbf{y}')),
\end{align*}
where at $(a)$ the expectation and the inner product commute because
of Bochner integrability of $(\mathbf{x},\mathbf{y})\mapsto K_{\mathbf{x}}\xi_{p_{|\mathbf{x}}}(\mathbf{y},\diamond)$
(see assumption \ref{enu:assume_bochner} in Theorem \ref{thm:pop_dp},
and \citet[Definition A.5.20]{SteChr2008}), at $(b)$ we use the
adjoint $K_{\mathbf{x}'}^{*}$ and the reproducing property i.e.,
$K_{\mathbf{x}'}^{*}K_{\mathbf{x}}=K(\mathbf{x},\mathbf{x}')=k(\mathbf{x},\mathbf{x}')I$,
\begin{align*}
 & h_{p}((\mathbf{x},\mathbf{y}),(\mathbf{x}',\mathbf{y}')):=\left\langle \xi_{p_{|\mathbf{x}}}(\mathbf{y},\diamond),\xi_{p_{|\mathbf{x}'}}(\mathbf{y}',\diamond)\right\rangle _{\mathcal{F}_{l}^{d_{y}}}\\
 & =l(\mathbf{y},\mathbf{y}')\mathbf{s}_{p}^{\top}(\mathbf{y}|\mathbf{x})\mathbf{s}_{p}(\mathbf{y}'|\mathbf{x}')+\sum_{i=1}^{d_{y}}\frac{\partial^{2}}{\partial y_{i}\partial y_{i}'}l(\mathbf{y},\mathbf{y}')\\
 & \phantom{=}+\mathbf{s}_{p}^{\top}(\mathbf{y}|\mathbf{x})\nabla_{\mathbf{y}'}l(\mathbf{y},\mathbf{y}')+\mathbf{s}_{p}^{\top}(\mathbf{y}'|\mathbf{x}')\nabla_{\mathbf{y}}l(\mathbf{y},\mathbf{y}'),
\end{align*}
and $\mathbf{s}_{p}(\mathbf{y}|\mathbf{x}):=\nabla_{\mathbf{y}}\log p(\mathbf{y}|\mathbf{x})$. 

\subsection{PROOF OF PROPOSITION \ref{prop:asymptotic_kssd}}

\label{subsec:proof_asymp_kssd}Define $\zeta_{1}:=\mathbb{V}\left[\mathbb{E}_{(\mathbf{x},\mathbf{y})\sim r_{xy}}H_{p}((\mathbf{x},\mathbf{y}),(\mathbf{x}',\mathbf{y}'))\right]$.
We only need to show that under $H_{0}$, $\dph$ is a degenerate
U-statistic i.e., $\zeta_{1}=0$, and under $H_{1}$, $\dph$ is non-degenerate
i.e., $\zeta_{1}>0$. Then, the asymptotic distributions in the two
cases follow from \citet[Section 5.5]{Ser2009}.

\paragraph{Case: $H_{0}$ is true}

\begin{align}
\mathbb{E}_{(\mathbf{x},\mathbf{y})\sim r_{xy}}H_{p}((\mathbf{x},\mathbf{y}),(\mathbf{x}',\mathbf{y}')) & =\mathbb{E}_{(\mathbf{x},\mathbf{y})\sim r_{xy}}\left\langle K_{\mathbf{x}}\xi_{p_{|\mathbf{x}}}(\mathbf{y},\diamond),K_{\mathbf{x'}}\xi_{p_{|\mathbf{x}'}}(\mathbf{y}',\diamond)\right\rangle _{\mathcal{F}_{K}}\nonumber \\
 & \stackrel{(a)}{=}\left\langle \mathbb{E}_{(\mathbf{x},\mathbf{y})\sim r_{xy}}K_{\mathbf{x}}\xi_{p_{|\mathbf{x}}}(\mathbf{y},\diamond),K_{\mathbf{x'}}\xi_{p_{|\mathbf{x}'}}(\mathbf{y}',\diamond)\right\rangle _{\mathcal{F}_{K}},\label{eq:kssd_half_eval_ustat}
\end{align}
where the interchange of the inner product and the expectation is
justified since $\mathbb{E}_{\mathbf{xy}}\|K_{\mathbf{x}}\xi_{p_{|\mathbf{x}}}(\mathbf{y},\diamond)\|_{\mathcal{F}_{K}}<\infty$
(Bochner integrability). But by Theorem \ref{thm:pop_dp} and (\ref{eq:kssd_pop}),
we have that $G_{p,r}:=\mathbb{E}_{(\mathbf{x},\mathbf{y})\sim r_{xy}}K_{\mathbf{x}}\xi_{p_{|\mathbf{x}}}(\mathbf{y},\diamond)=\mathbf{0}$.
So, $\zeta_{1}=0$ and the result under $H_{0}$ follows from \citet[Section 5.5.2]{Ser2009}.

\paragraph{Case: $H_{1}$ is true}

From (\ref{eq:kssd_pop_ustat}), it can be seen that 
\begin{align*}
\eqref{eq:kssd_half_eval_ustat} & =\mathbb{E}_{\mathbf{xy}}k(\mathbf{x},\mathbf{x}')h_{p}((\mathbf{x},\mathbf{y}),(\mathbf{x}',\mathbf{y}')):=t(\mathbf{x}',\mathbf{y}').
\end{align*}
Since $\zeta_{1}=\mathbb{V}[t(\mathbf{x},\mathbf{y})]$, it suffices
to show that $t$ is not a constant function. To see this, note that
the kernel $k$ is $C_{0}$-universal and cannot be a constant function.
The function $h_{p}$ (see (\ref{eq:unsmoothed_ustat_kernel})) includes
the kernel $l$ which is also $C_{0}$-universal. Therefore, $t$
is not a constant function and $\zeta_{1}>0$. We get the asymptotic
normality from the result in \citet[Section 5.5.1]{Ser2009}.

\subsection{PROOF OF THEOREM \ref{thm:fscd_pop}}

\label{subsec:proof_tpvdiv}Recall the proposition from the main text:\tpvstatdiv*
\begin{proof}
Recall that $\tpv(r):=\frac{1}{Jd_{y}}\sum_{i=1}^{J}\|G_{p,r}(\mathbf{v}_{i})\|_{\mathcal{F}_{l}^{d_{y}}}^{2}.$
If $H_{0}$ is true, then $G_{p,r}=\mathbf{0}$ by Theorem \ref{thm:pop_dp}.
As a result, $\tpv(r)=0$. Now suppose that $H_{1}$ is true. We first
show that $\tilde{G}(\mathbf{v}):=\|G_{p,r}(\mathbf{v})\|_{\mathcal{F}_{l}^{d_{y}}}^{2}$
is a real analytic function. Consider 
\begin{align*}
\bar{G}(\mathbf{v},\mathbf{v'}) & =\mathbb{E}_{\mathbf{x}\mathbf{y}}\mathbb{E}_{\mathbf{x'y'}}k(\mathbf{x},\mathbf{v})k(\mathbf{x}',\mathbf{v}')h_{p}((\mathbf{x},\mathbf{y}),(\mathbf{x}',\mathbf{y}'))\\
 & =\mathbb{E}_{\mathbf{x,x'}}\tilde{k}[(\mathbf{x},\mathbf{x}'),(\mathbf{v},\mathbf{v}')]\tilde{h}_{p}(\mathbf{x},\mathbf{x}'),
\end{align*}
where $\tilde{k}[(\mathbf{x},\mathbf{x}'),(\mathbf{v},\mathbf{v}')]:=k(\mathbf{x},\mathbf{v})k(\mathbf{x}',\mathbf{v}')$
and $\tilde{h}_{p}(\mathbf{x},\mathbf{x'})=\mathbf{\mathbb{E}_{y|\mathbf{x}}\mathbb{E}_{\mathbf{y'}|\mathbf{x'}}}h_{p}((\mathbf{x},\mathbf{y}),(\mathbf{x}',\mathbf{y}'))$.
Note that $\tilde{h}_{p}(\mathbf{x},\mathbf{x}')=\langle\mathbf{g}_{p,r}(\diamond|\mathbf{x}),\mathbf{g}_{p,r}(\diamond|\mathbf{x}')\rangle_{\mathcal{F}_{l}^{d_{y}}}$,
and thus we have 
\begin{align*}
\mathbb{E}_{\mathbf{x,x'}} & \langle\mathbf{g}_{p,r}(\diamond|\mathbf{x}),\mathbf{g}_{p,r}(\diamond|\mathbf{x}')\rangle_{\mathcal{F}_{l}^{d_{y}}}^{2}\leq\left(\mathbf{\mathbb{E}_{\mathbb{\mathbf{x}}}}\lVert\mathbf{g}_{p,r}(\diamond|\mathbf{x})\rVert_{\mathcal{F}_{l}^{d_{y}}}^{2}\right)^{2}.
\end{align*}
The RHS is finite by Assumption 3 in Theorem \ref{thm:pop_dp}, and
so $\tilde{h}_{p}\in L^{2}(\mathcal{X\times\mathcal{X}},r_{x}\otimes r_{x})$.
Therefore, $\bar{G}$ is given by the integral transform of $\tilde{h}_{p}$
with respect to the kernel $\tilde{k}$, which implies that $\bar{G}$
is an element of the RKHS of $\tilde{k}$ \citep[Theorem 4.26]{SteChr2008}.
Since the product of real analytic functions is real analytic, consequently
for any $(\mathbf{v},\mathbf{v}')$, $(\mathbf{z},\mathbf{z}')\mapsto\tilde{k}[\mathbf{(z},\mathbf{z}'),\mathbf{(v,\mathbf{v')}}]$
is real analytic and bounded by our assumption. Thus, by Lemma \ref{lem:rkhs_analytic},
$\bar{G}(\mathbf{v},\mathbf{v}')$ is analytic. From (\ref{eq:norm_rkhs_witness}),
we have $\tilde{G}(\mathbb{\mathbf{v}})=\bar{G}(\mathbf{v},\mathbf{v})$;
hence $\tilde{G}$ is analytic and not a zero function by Theorem
\ref{thm:pop_dp}. Since the zero set of $\tilde{G}(\mathbf{v})$,
$\{\mathbf{v}'\in\mathcal{X}\mid\tilde{G}(\mathbf{v}')=0\}$, has
zero Lebesgue measure \citep{Mit2015}, we have that $\eta$-almost
surely $\tilde{G}(\mathbf{v})>0$ for any $\mathbf{v}\sim\eta$, and
the result follows.
\end{proof}

\subsection{PROOF OF PROPOSITION \ref{prop:fscd_pop_ustat}}

From (\ref{eq:fscd_pop}), we first rewrite $\|G_{p,r}(\mathbf{v})\|_{\mathcal{F}_{l}^{d_{y}}}^{2}$
as
\begin{align}
 & \|G_{p,r}(\mathbf{v})\|_{\mathcal{F}_{l}^{d_{y}}}^{2}=\big\|\left[\mathbb{E}_{(\mathbf{x},\mathbf{y})\sim r_{xy}}K_{\mathbf{x}}\xi_{p_{|\mathbf{x}}}(\mathbf{y},\diamond)\right](\mathbf{v})\big\|_{\mathcal{F}_{l}^{d_{y}}}^{2}\nonumber \\
 & \stackrel{(a)}{=}\big\|\mathbb{E}_{(\mathbf{x},\mathbf{y})\sim r_{xy}}K(\mathbf{x},\mathbf{v})\xi_{p_{|\mathbf{x}}}(\mathbf{y},\diamond)\big\|_{\mathcal{F}_{l}^{d_{y}}}^{2}\nonumber \\
 & \stackrel{(b)}{=}\mathbb{E}_{\mathbf{x}\mathbf{y}}\mathbb{E}_{\mathbf{x'y'}}k(\mathbf{x},\mathbf{v})k(\mathbf{x}',\mathbf{v})h_{p}((\mathbf{x},\mathbf{y}),(\mathbf{x}',\mathbf{y}'))\label{eq:norm_rkhs_witness}
\end{align}
where at $(a)$ we use $(K_{\mathbf{x}}f)(\mathbf{v})=K(\mathbf{x},\mathbf{v})f$
for $f\in\mathcal{F}_{l}^{d_{y}}$, and at $(b)$ we use $h_{p}((\mathbf{x},\mathbf{y}),(\mathbf{x}',\mathbf{y}'))=\left\langle \xi_{p_{|\mathbf{x}}}(\mathbf{y},\diamond),\xi_{p_{|\mathbf{x}'}}(\mathbf{y}',\diamond)\right\rangle _{\mathcal{F}_{l}^{d_{y}}}$
as in (\ref{eq:kssd_pop_ustat}). It follows from (\ref{eq:fscd_pop})
that 
\begin{align*}
\tpv(r) & =\mathbb{E}_{\mathbf{x}\mathbf{y}}\mathbb{E}_{\mathbf{x'y'}}\hbpv((\mathbf{x},\mathbf{y}),(\mathbf{x}',\mathbf{y}')),
\end{align*}
where $(\mathbf{x},\mathbf{y}),(\mathbf{x}',\mathbf{y}')$ are i.i.d.
random variables following $r_{xy}$, 
\[
\hbpv((\mathbf{x},\mathbf{y}),(\mathbf{x}',\mathbf{y}')):=\frac{1}{d_{y}}\overline{k}_{V}(\mathbf{x},\mathbf{x}')h_{p}((\mathbf{x},\mathbf{y}),(\mathbf{x}',\mathbf{y}')),
\]
 and $\overline{k}_{V}(\mathbf{x},\mathbf{x}'):=\frac{1}{J}\sum_{i=1}^{J}k(\mathbf{x},\mathbf{v}_{i})k(\mathbf{x}',\mathbf{v}_{i})$
is a kernel that depends on $V$.
\begin{prop}[Asymptotic distributions of $\tpvh$]
 \label{prop:fscd_asymptotics}Assume that $\mathbb{E}\overline{k}_{V}^{2}(\mathbf{x},\mathbf{x}')h_{p}^{2}((\mathbf{x},\mathbf{y}),(\mathbf{x}',\mathbf{y}'))<\infty$.
The following statements hold.
\begin{enumerate}
\item If $\sigma_{V}^{2}:=4\mathbb{V}[\mathbb{E}_{\mathbf{xy}}[\hbpv((\mathbf{x},\mathbf{y}),(\mathbf{x}',\mathbf{y}'))]]>0$,
then $\sqrt{n}\left(\tpvh-\tpv(r)\right)\stackrel{d}{\to}\mathcal{N}(0,\sigma_{V}^{2})$;
\item If $\sigma_{V}^{2}=0$, then $n\tpvh\stackrel{d}{\to}\sum_{j=1}^{\infty}\lambda_{j}(\chi_{1j}^{2}-1)$,
where $\{\chi_{1j}^{2}\}_{j}$ are independent $\chi_{1}^{2}$ random
variables, $\lambda_{j}$ are eigenvalues of the operator $A$ defined
as $(A\varphi)(\mathbf{z})=\int\hbpv(\mathbf{z},\mathbf{z}')\varphi(\mathbf{z}')r_{xy}(\mathbf{z}')\thinspace\mathrm{d}\mathbf{z}'$
for non-zero $\varphi$, $\mathbf{z}:=(\mathbf{x},\mathbf{y})$ and
$\mathbf{z}':=(\mathbf{x}',\mathbf{y}')$.
\end{enumerate}
\end{prop}

\section{ILLUSTRATION OF THE FSCD POWER CRITERION}

\label{sec:toy_powcri} To complement Figure \ref{fig:toy_illus},
in this section, we illustrate the behavior of the power criterion
of the FSCD as a function of the test location on a number of one-dimensional
problems. These are shown in Figure \ref{fig:fscd_powcri_toys} where
observed data from $r$ are shown in blue. In all cases, the two kernels
$k$ and $l$ are set to Gaussian kernels.

\begin{figure}[h]
\centering
\vspace{-3mm}
\subfloat[Quadratic Gaussian Model (Gaussian $r_x$)\label{fig:toy_qgm_gauss}]{
\includegraphics[width=0.26\linewidth]{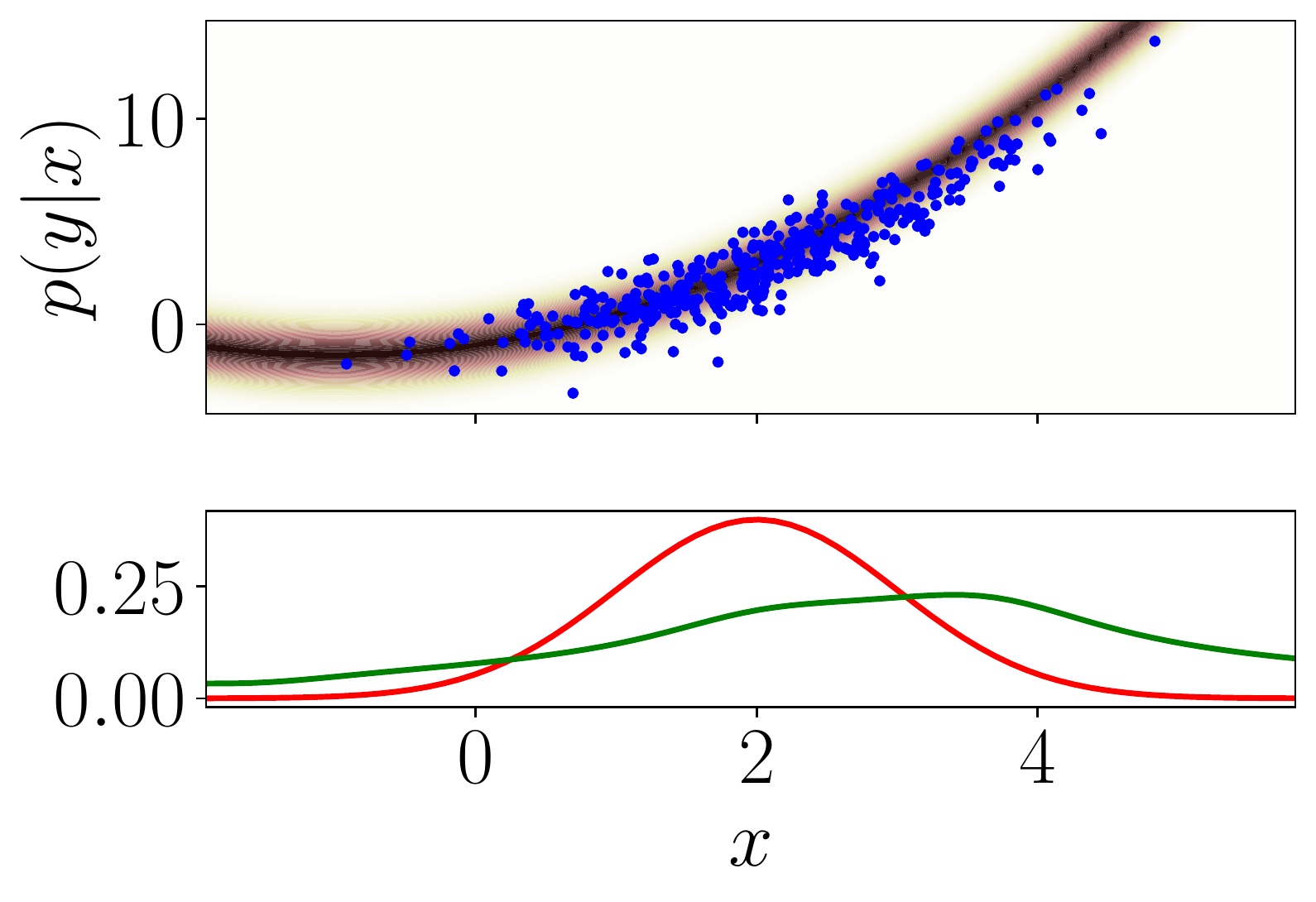}
}
\subfloat[Hetero. Gaussian Model\label{fig:toy_hgm}]{
\includegraphics[width=0.26\linewidth]{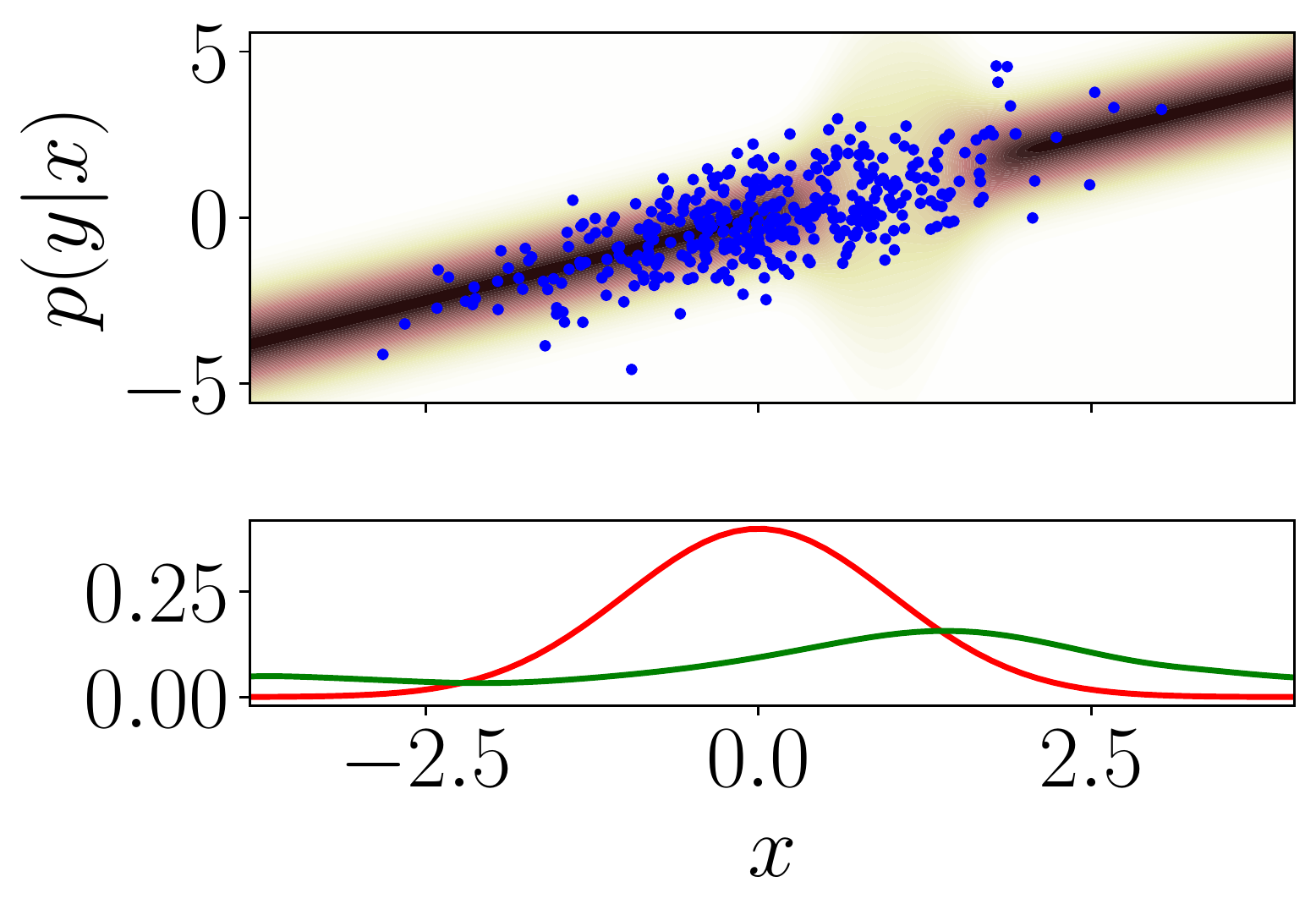}
}
\subfloat[Quadratic Gaussian Model (uniform $r_x$) \label{fig:toy_qgm_unif}]{
\includegraphics[width=0.26\linewidth]{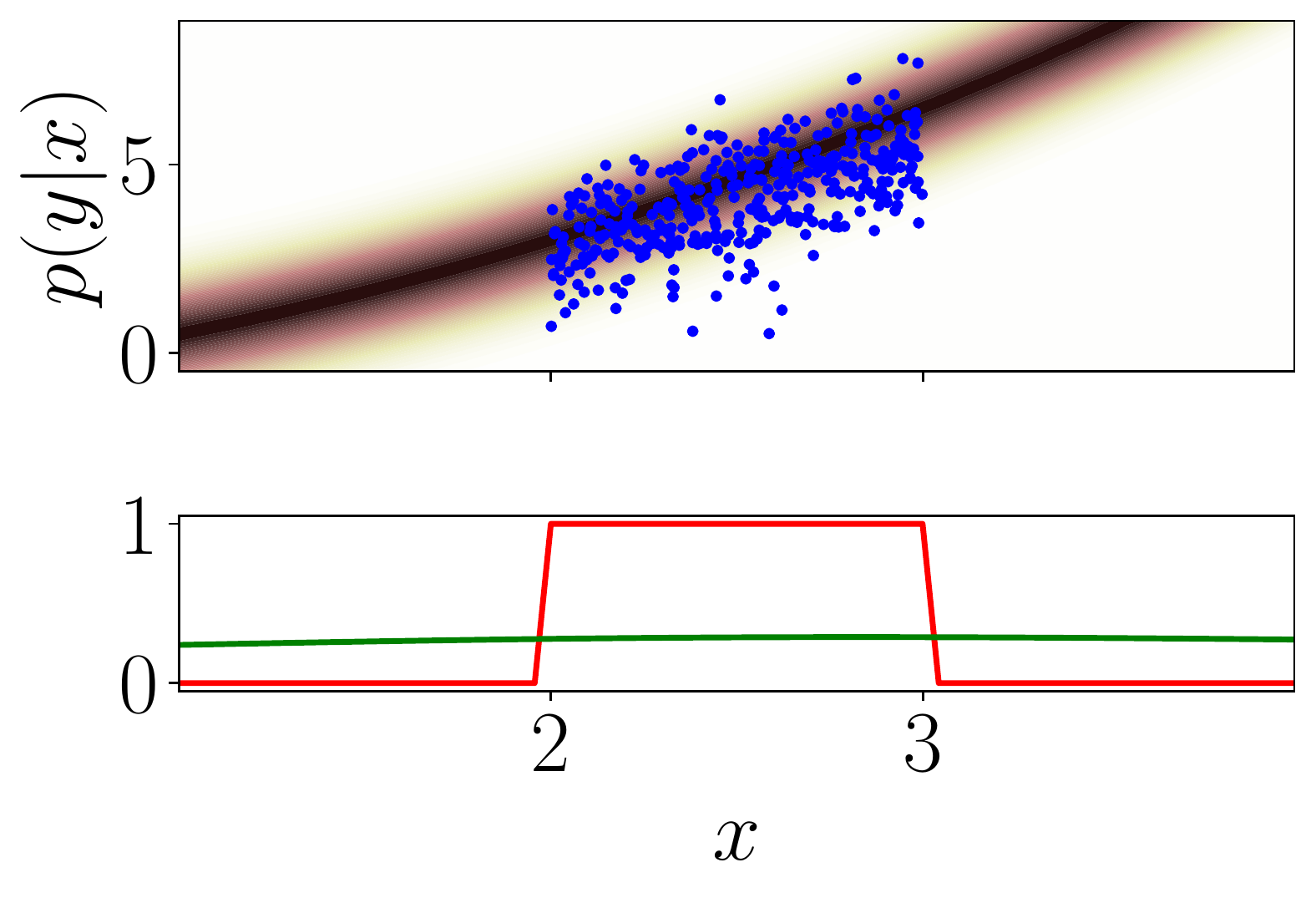}
\includegraphics[width=0.19\linewidth]{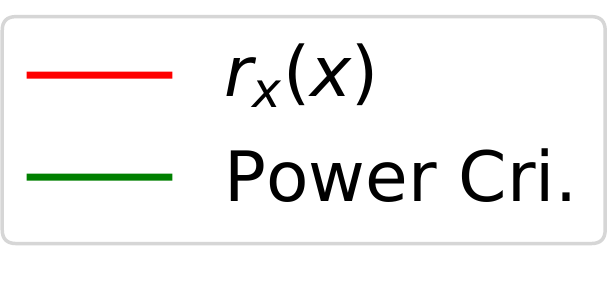} 
}\caption{The power criterion as a function of $x$ on a number of synthetic
one-dimensional problems. \label{fig:fscd_powcri_toys}}
\end{figure}
In Figure \ref{fig:toy_qgm_gauss}, $p(y|x):=\mathcal{N}\left(x+0.5x^{2}-1,1\right),r(y|x):=\mathcal{N}\left(x+0.4x^{2}-1,1\right),$
and $r_{x}(x):=\mathcal{N}(2,1)$. Here, the difference between the
data generating distribution $r$ and the model $p$ is the in the
coefficients of the second-order term in the mean, which differ only
slightly. We observe that the power criterion function is non-zero
almost everywhere.

In Figure \ref{fig:toy_hgm}, $p(y|x):=\mathcal{N}\left(x,\sigma^{2}(x)\right)$
where the variance function is $\sigma^{2}(x):=1+8\exp\left(-\frac{(x-1)^{2}}{2\times0.3^{2}}\right)$.
This is a linear regression model with heteroscedastic noise, and
is similar to the HGM model considered in Section \ref{sec:experiments}.
We set $r(y|x):=\mathcal{N}(x,1)$ and $r_{x}(x):=\mathcal{N}(0,1)$.
The problem is designed so that the model and the true conditional
density have the same (conditional) mean, but differ locally in the
conditional variance i.e., at $x=1$. We observe that the power criterion
function indeed has a peak around $x=1$, indicating that it is sensitive
to local differences. Note that theoretically the power criterion
function is non-zero almost everywhere (but could be arbitrarily close
to zero).

In Figure \ref{fig:toy_qgm_unif}, we consider the same $p$ and $r$
as specified in Figure \ref{fig:toy_qgm_gauss}, but change $r_{x}$
to be a uniform distribution defined on $[2,3]$. Since $p$ is different
from $r$, the power criterion function is non-zero almost everywhere
(implied by Theorem \ref{thm:fscd_pop}), and this is indeed the case.
We note that this statement holds true regardless of $r_{x}$, as
evident in Figure \ref{fig:toy_qgm_unif} (cf. Figure \ref{fig:toy_qgm_gauss}).
In particular, the support of $r_{x}$ may not cover the whole domain
$\mathcal{X}$.

\section{NYC TAXI DATA EXPERIMENT}

\label{sec:nyc_taxi_experiment}

\subsection{TRAINING OF THE MIXTURE DENSITY NETWORK}

Here, we describe technical details of the Mixture Density Network
(MDN) used in the NYC taxi data experiment\footnote{Our implementation is lightly based on public code at \url{https://github.com/sagelywizard/pytorch-mdn}.}
for estimating  the conditional probability of a drop-off location
given a pick-up location. The NYC taxi dataset is available at \url{https://www1.nyc.gov/site/tlc/about/tlc-trip-record-data.page}.
An MDN specifies a conditional density model of the form
\begin{align*}
p(\mathbf{y}|\mathbf{x}) & =\sum_{i=1}^{C}\pi_{i}(\mathbf{x})\mathcal{N}\left(\mathbf{y}\mid\boldsymbol{\mu}_{i}(\mathbf{x}),\mathrm{diag}\left(\sigma_{i,1}^{2}(\mathbf{x}),\ldots,\sigma_{i,d_{y}}^{2}(\mathbf{x})\right)\right),
\end{align*}
where $C$ is the number of Gaussian components, $\mathbf{x}\in\mathbb{R}^{d_{x}},\mathbf{y}\in\mathbb{R}^{d_{y}}$
and $\mathrm{diag}(\mathbf{s})$ constructs a diagonal matrix with
the diagonal entries given by $\mathbf{v}$. In our problem,\textbf{
$\mathbf{x}$ }(pick-up location) and $\mathbf{y}$ (drop-off location)
contain latitude/longitude coordinates; so, $d_{x}=d_{x}=2$. The
mixing proportion function $\boldsymbol{\pi}(\mathbf{x}):=(\pi_{1}(\mathbf{x}),\ldots,\pi_{C}(\mathbf{x}))$,
the mean function $\boldsymbol{\mu}(\mathbf{x}):=(\boldsymbol{\mu}_{1}(\mathbf{x}),\ldots,\boldsymbol{\mu}_{C}(\mathbf{x}))^{\top}\in\mathbb{R}^{C\times d_{y}}$,
and the variance function $\boldsymbol{\sigma}^{2}(\mathbf{x}):=\left(\sigma_{i,j}^{2}(\mathbf{x})\right)_{i,j}\in\mathbb{R}_{+}^{C\times d_{y}}$
for $i\in\{1,\ldots,C\},j\in\{1,2\}$ depend on $\mathbf{x}$ and
are specified by neural networks. The network architecture is as follows:

\raggedbottom
\begin{minipage}{\linewidth}
\centering
\begin{tabular}{ p{1.25in} p{.85in} *4{p{.75in}}}\toprule[1.5pt]
\bf Layer  $\downarrow$         & \bf Input & \bf Output \\ \midrule \midrule
Linear          & $d_x=2$ & 128            \\
Batch normalization &  -             &      -           \\
ReLU activation &  -              &      -           \\
Linear          & 128           & 64             \\
Batch normalization        &  -              &   -              \\
ReLU activation &  -              &  -               \\
Linear          & 64            & $C=20$             \\ 
Softmax          &   -             &     -          \\ \bottomrule
\end{tabular}
\captionof{table}{Network architecture for  $\boldsymbol{\pi}$.}
\end{minipage}
 
\begin{minipage}{\linewidth}
\centering
\begin{tabular}{ p{1.25in} p{.85in} *4{p{.75in}}}\toprule[1.5pt]
\bf Layer     $\downarrow$      & \bf Input & \bf Output \\ \midrule \midrule
Linear          &  $d_x=2$          & 128            \\
Batch normalization        &      -          &  -               \\
ReLU activation &   -             &  -               \\
Linear          & 128           & 64             \\
Batch normalization        &    -            &   -              \\
Linear          & 64            & $C \times d_y$             \\ \bottomrule
\end{tabular}
\captionof{table}{Network architecture for  $\boldsymbol{\mu}$.}
\end{minipage}
 
\begin{minipage}{\linewidth}
\centering
\begin{tabular}{ p{1.25in} p{.85in} *4{p{.75in}}}\toprule[1.5pt]
\bf Layer   $\downarrow$        & \bf Input & \bf Output \\ \midrule \midrule
Linear          & $d_x=2$         & 128            \\
Batch normalization        &    -            &  -               \\
ReLU activation &    -            &   -              \\
Linear          & 128           & 64             \\
Batch normalization        &   -             &  -               \\
Linear          & 64            & $C \times d_y$            \\ \bottomrule
\end{tabular}
\captionof{table}{Network architecture for $\boldsymbol{\sigma}^2$.}
\end{minipage}

We train the model on five million trip records of New York's yellow
cabs from January 2015 using $C=20$ Gaussian components. Only trips
with pick-up and drop-off locations within or close to Manhattan are
used. The training objective function is the negative log likelihood.
We train the model for three epochs using Adam \citep{KinBa2014}
as the optimization procedure, with a minibatch size of 2000, and
a learning rate of 0.001.

\subsection{POWER CRITERION OF THE FSCD}

In this section, we show more results akin to Figure \ref{fig:nyc_power_criterion}.
Here, we sample a number of candidate test location $\mathbf{v}$'s,
and evaluate the FSCD power criterion (see Section \ref{sec:optimized_fscd})
at each of these locations separately. The test location is denoted
by $\blacktriangle$ in the following figures. We use the same setting
as used to produce Figure \ref{fig:nyc_power_criterion}. It is worth
reiterating that the same sample $\{(\mathbf{x}_{i},\mathbf{y}_{i})\}_{i=1}^{12000}$
is used to compute the power criterion in all cases. Each of the following
figures corresponds to one realization of $\mathbf{v}$. Only sample
points $\{(\mathbf{x}_{i},\mathbf{y}_{i})\mid i=1,\ldots,12000\text{ and }\mathbf{x}_{i}\text{ is close to }\text{\textbf{v }}\}$
are shown (in blue), and not the full sample. The trained MDN's density
function $p(\cdot|\mathbf{v})$ is shown as a contour plot (in black).

\newcommand{\imgw}{40mm}

\includegraphics[width=\imgw]{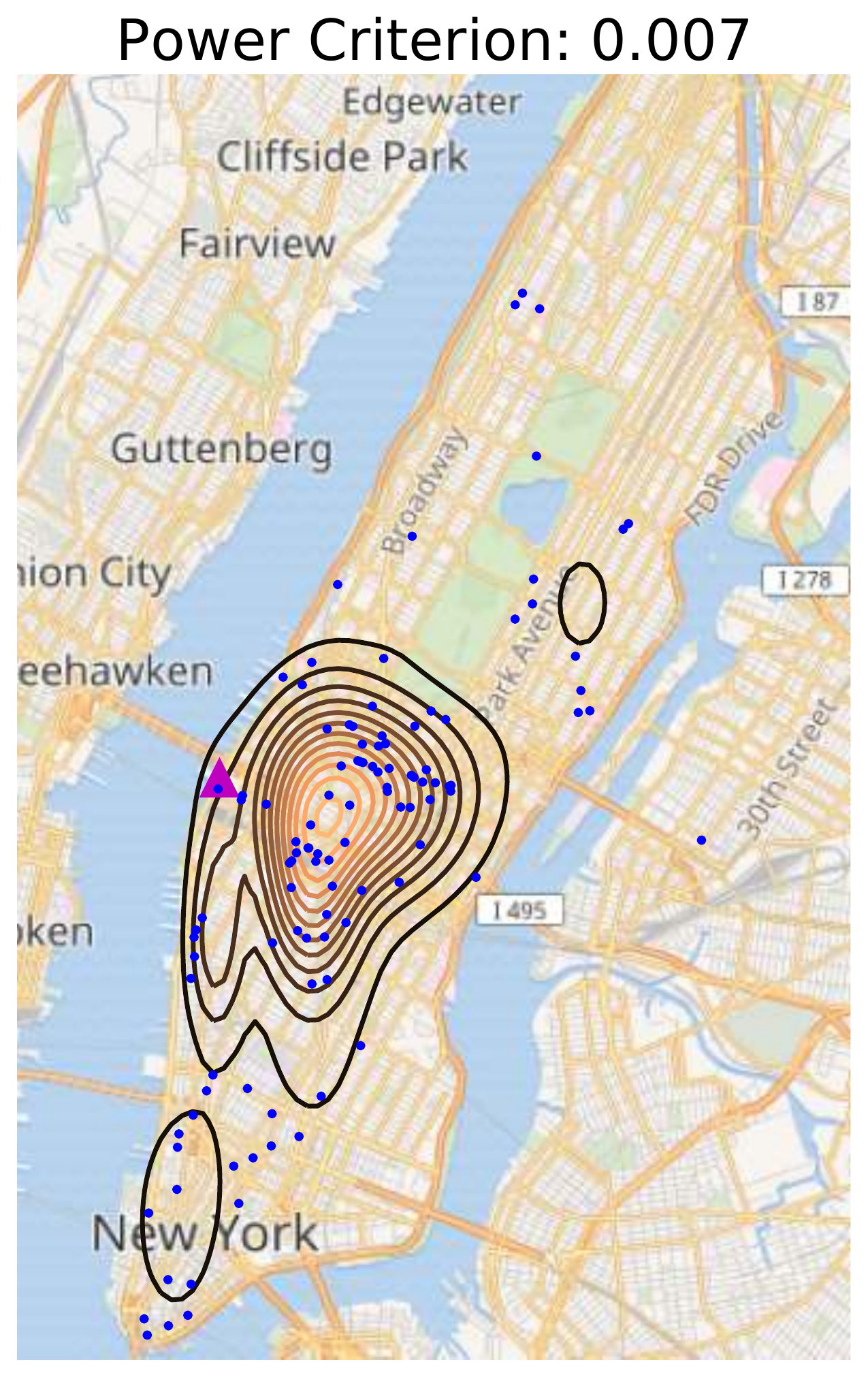}
\includegraphics[width=\imgw]{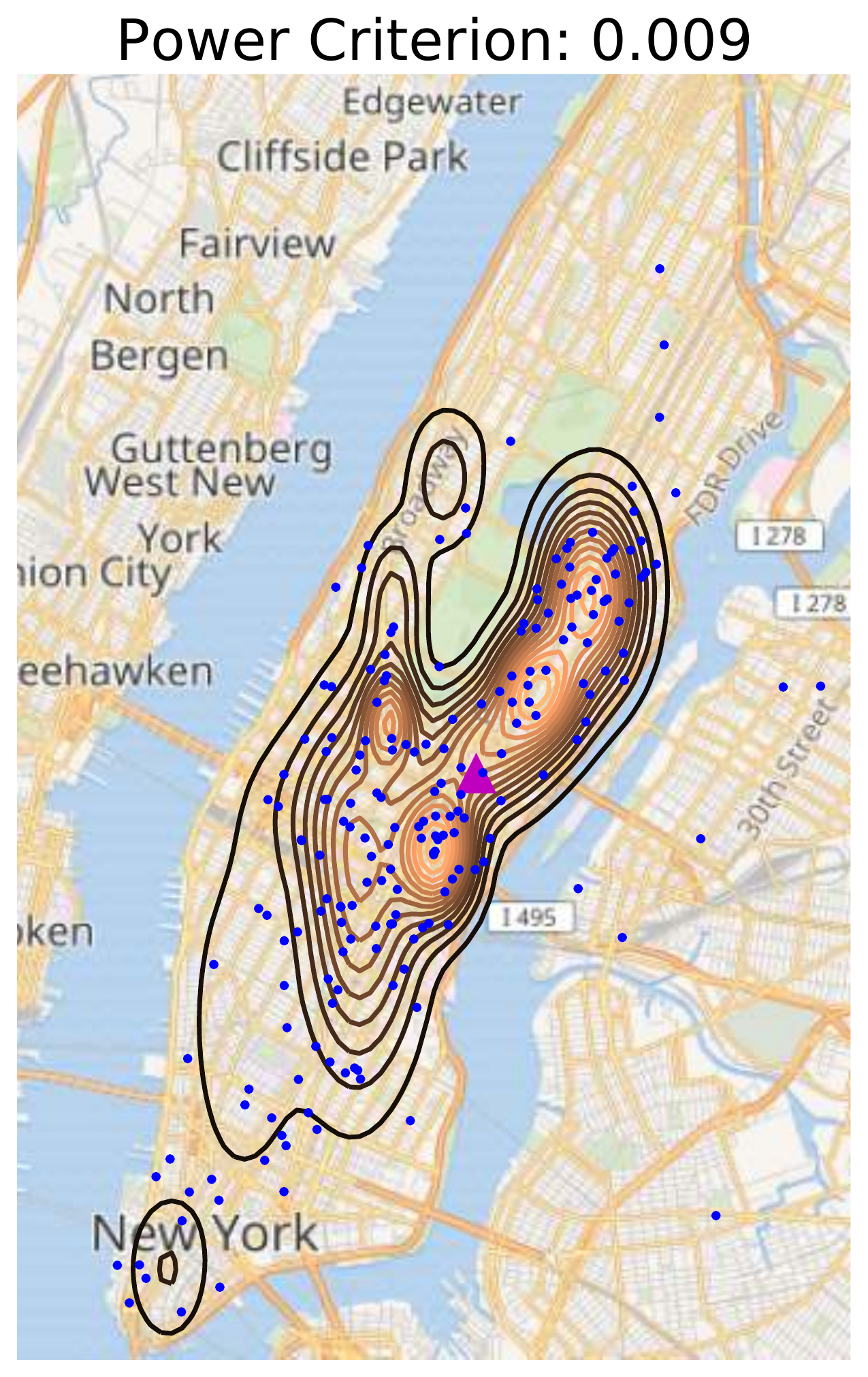}
\includegraphics[width=\imgw]{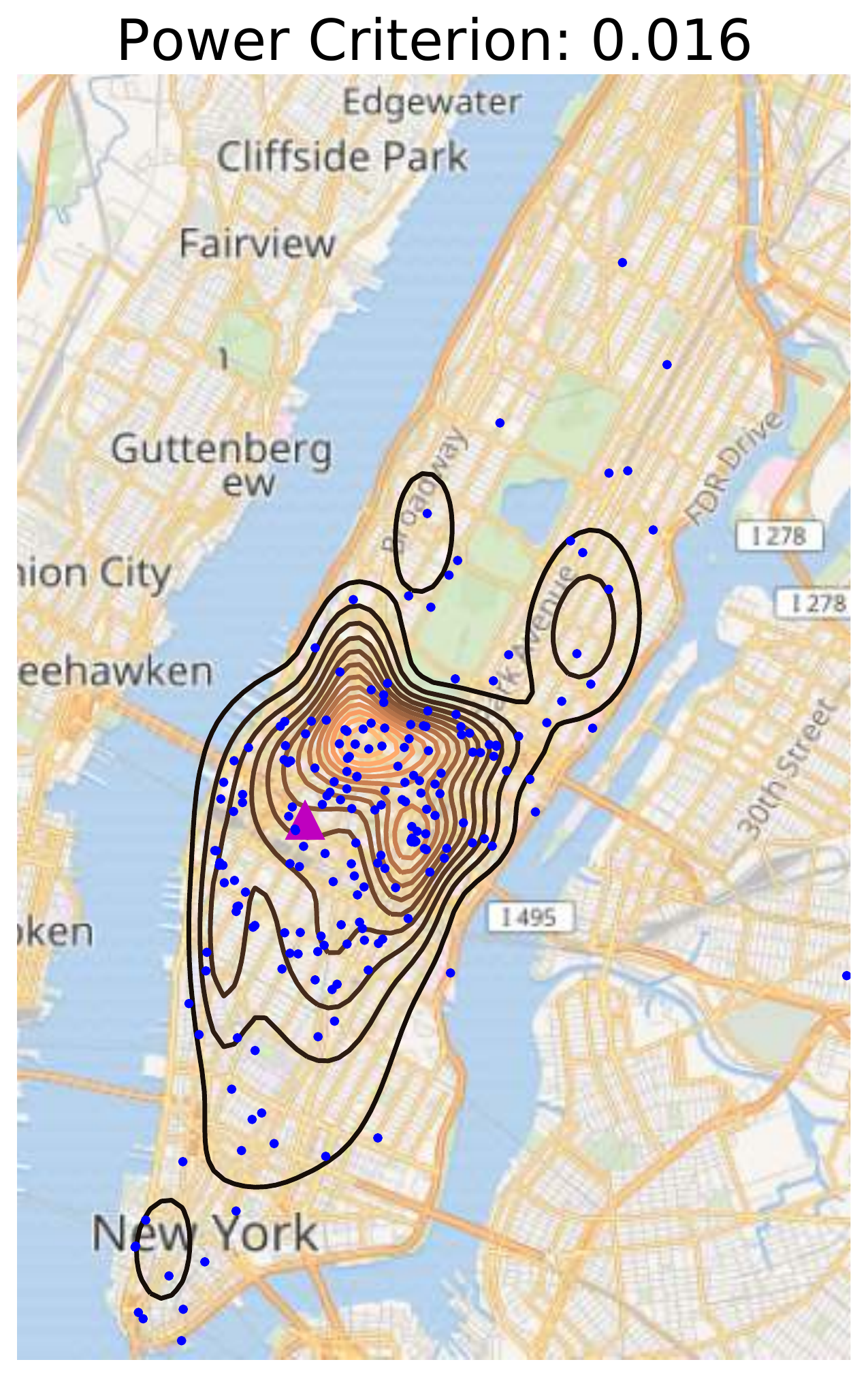}
\includegraphics[width=\imgw]{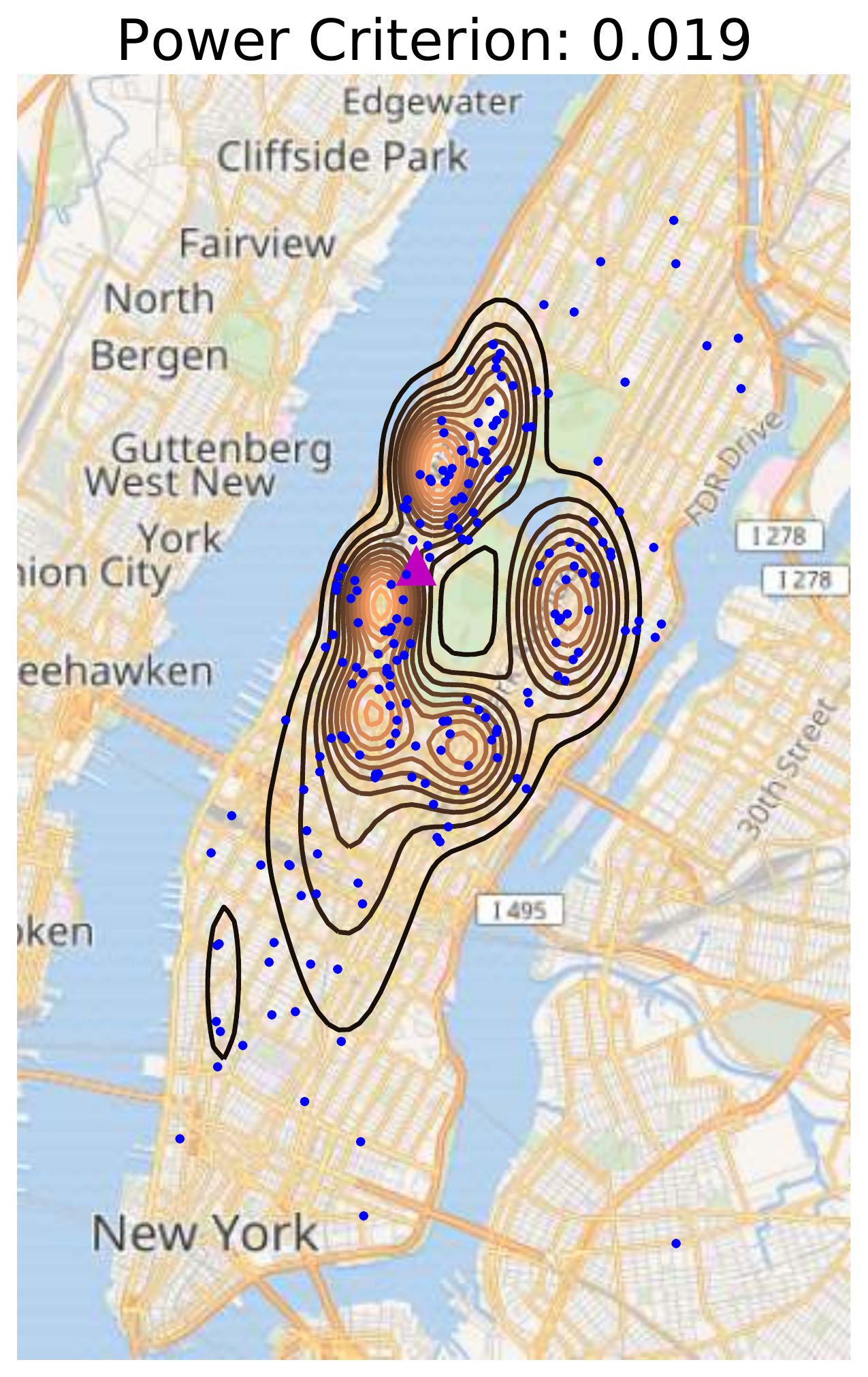}
\includegraphics[width=\imgw]{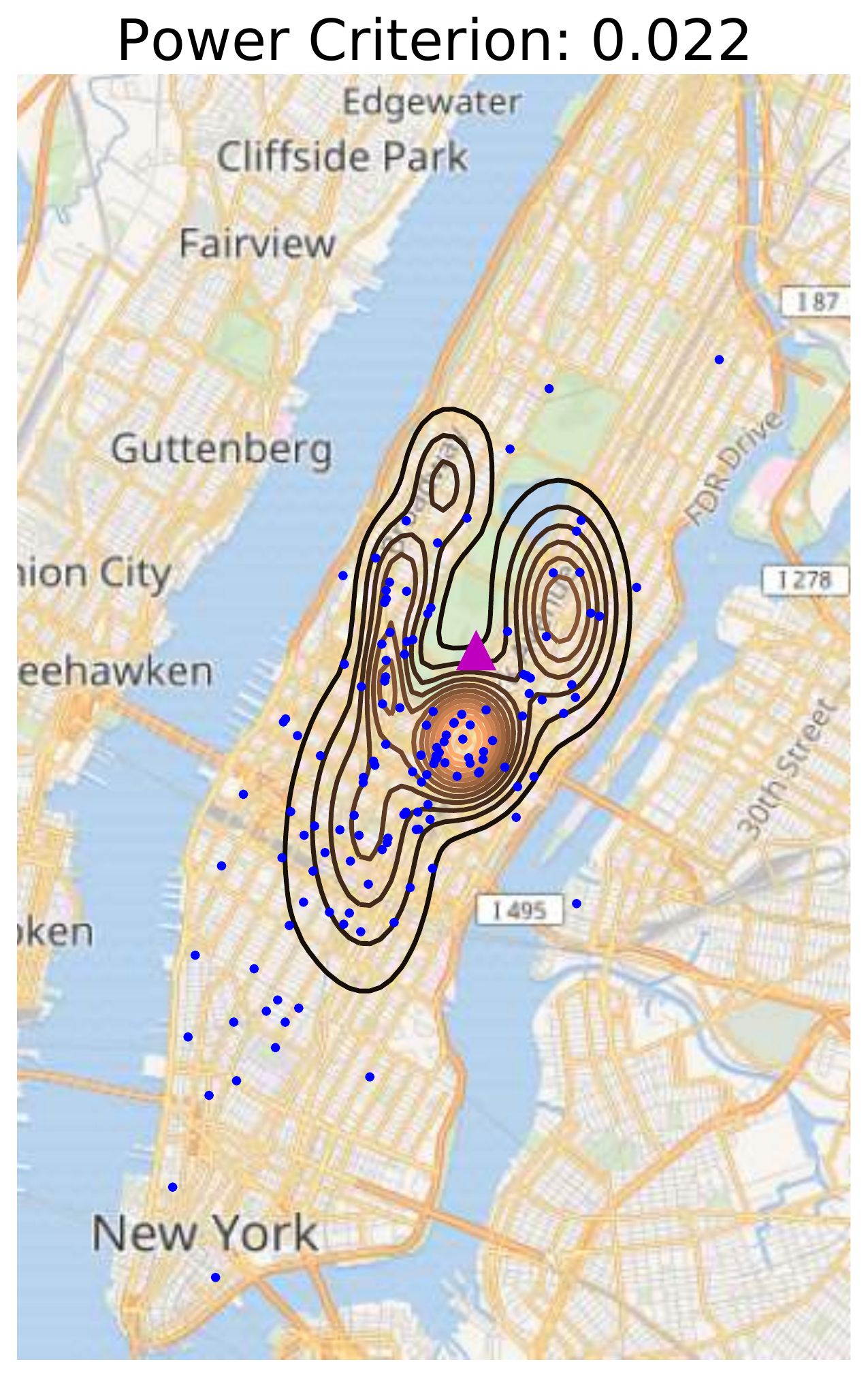}
\includegraphics[width=\imgw]{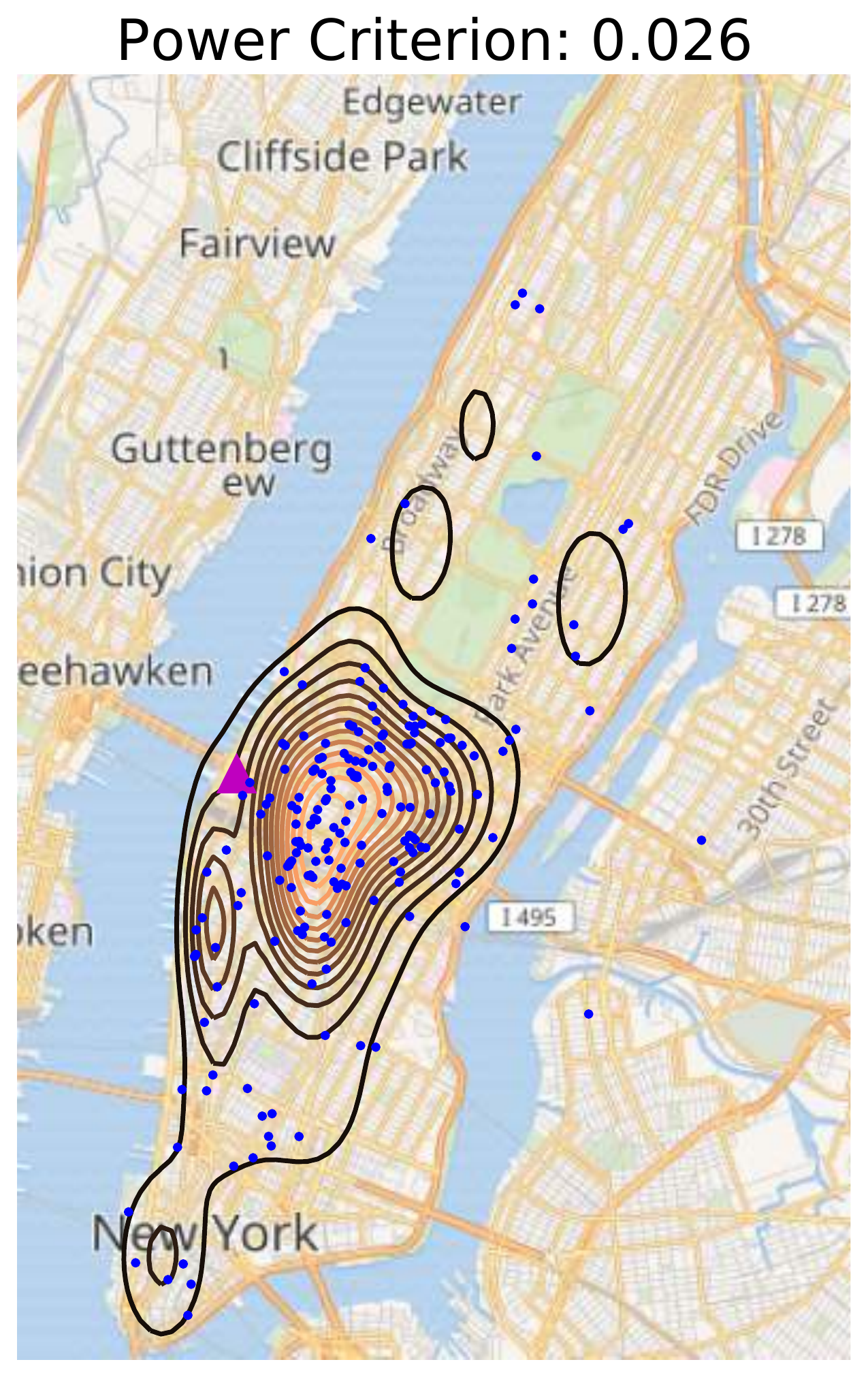}
\includegraphics[width=\imgw]{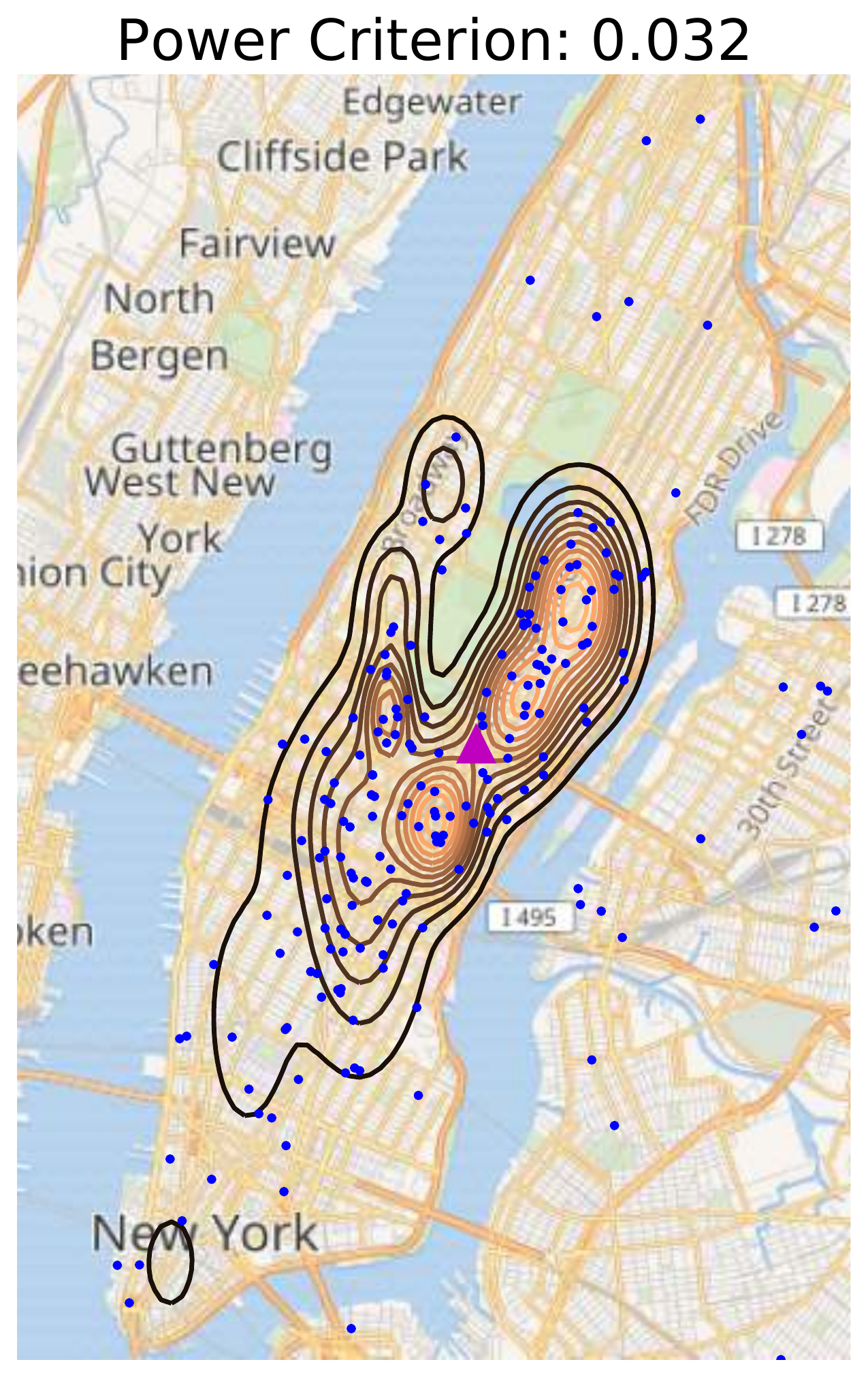}
\includegraphics[width=\imgw]{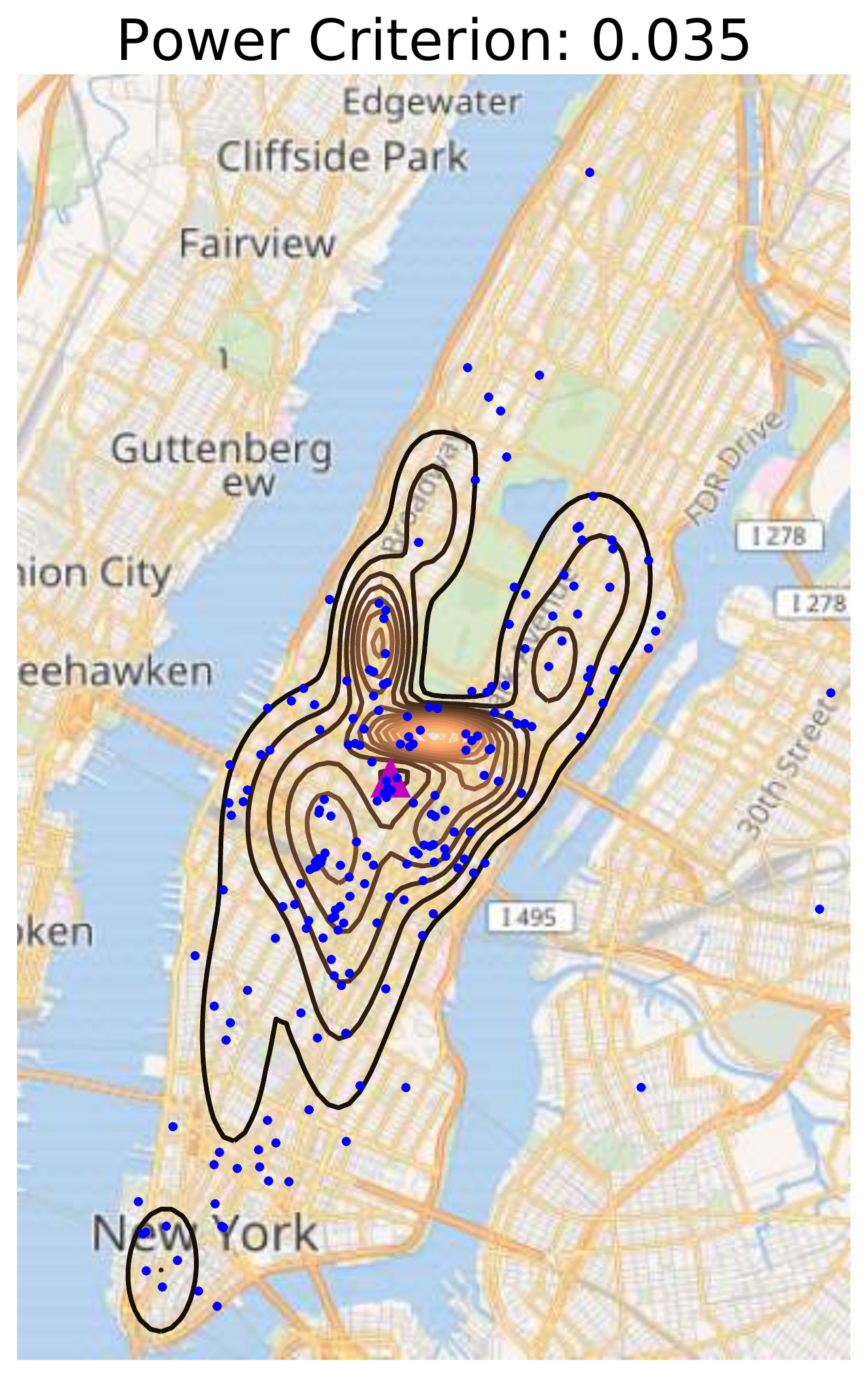}
\includegraphics[width=\imgw]{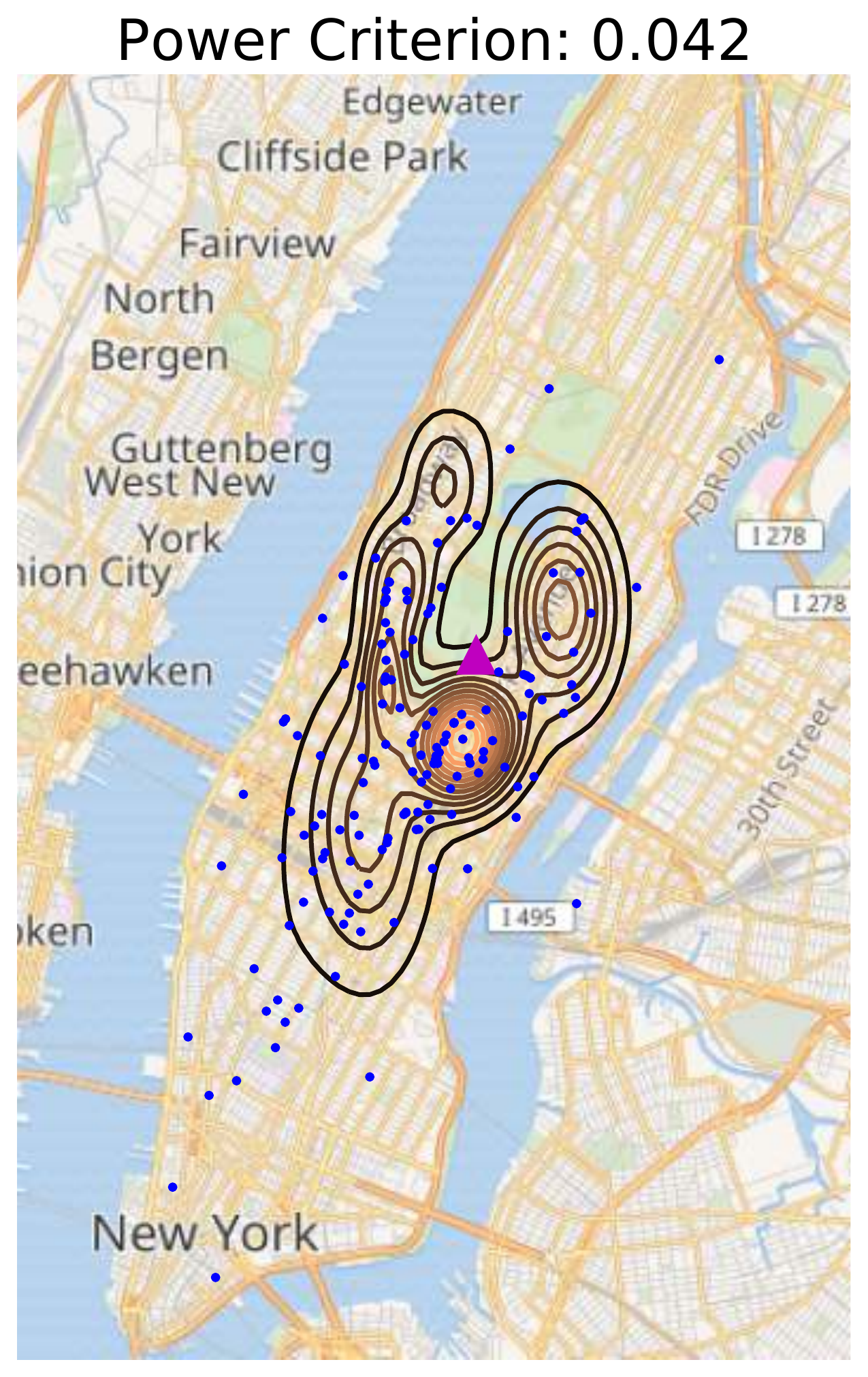}
\includegraphics[width=\imgw]{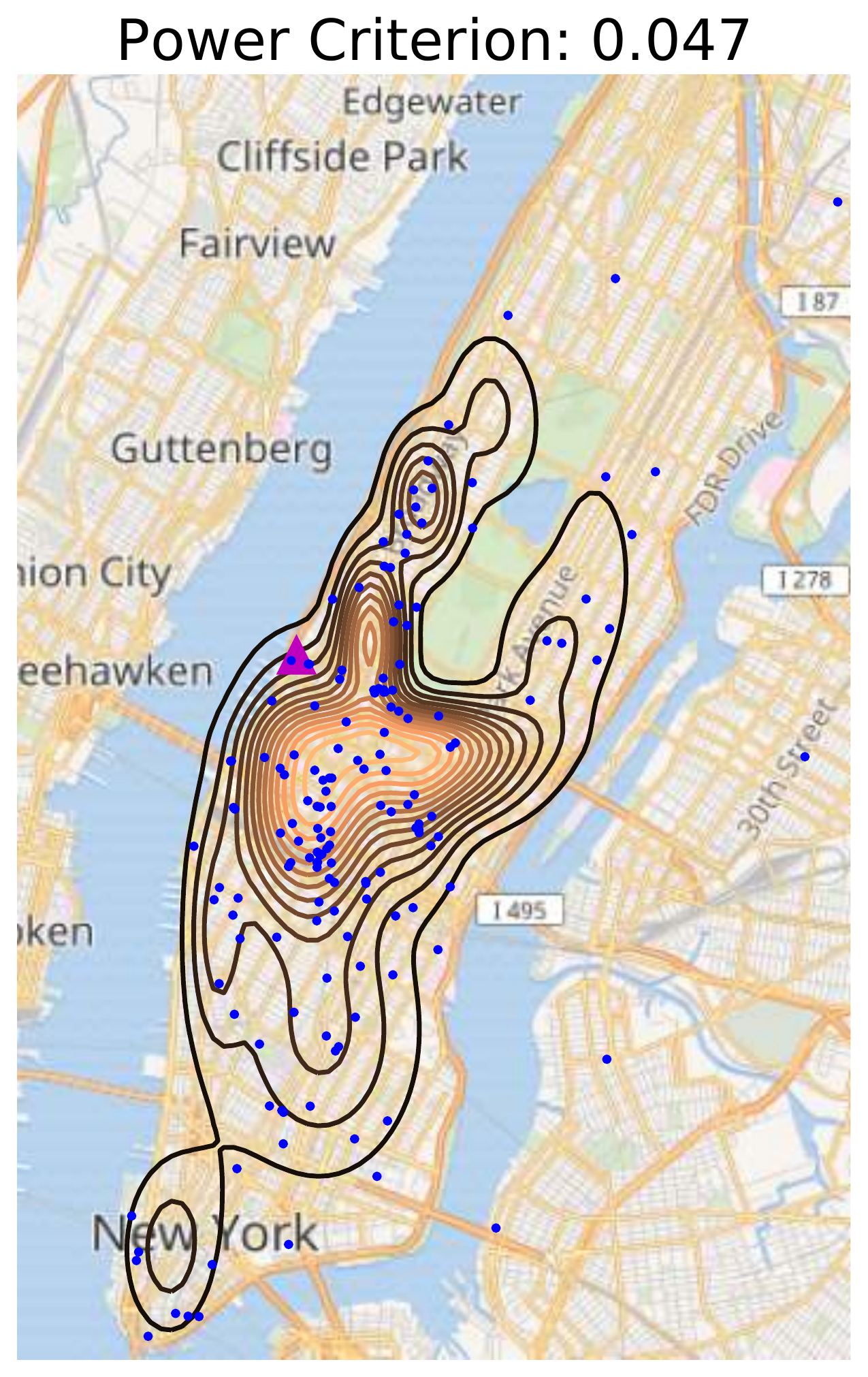}
\includegraphics[width=\imgw]{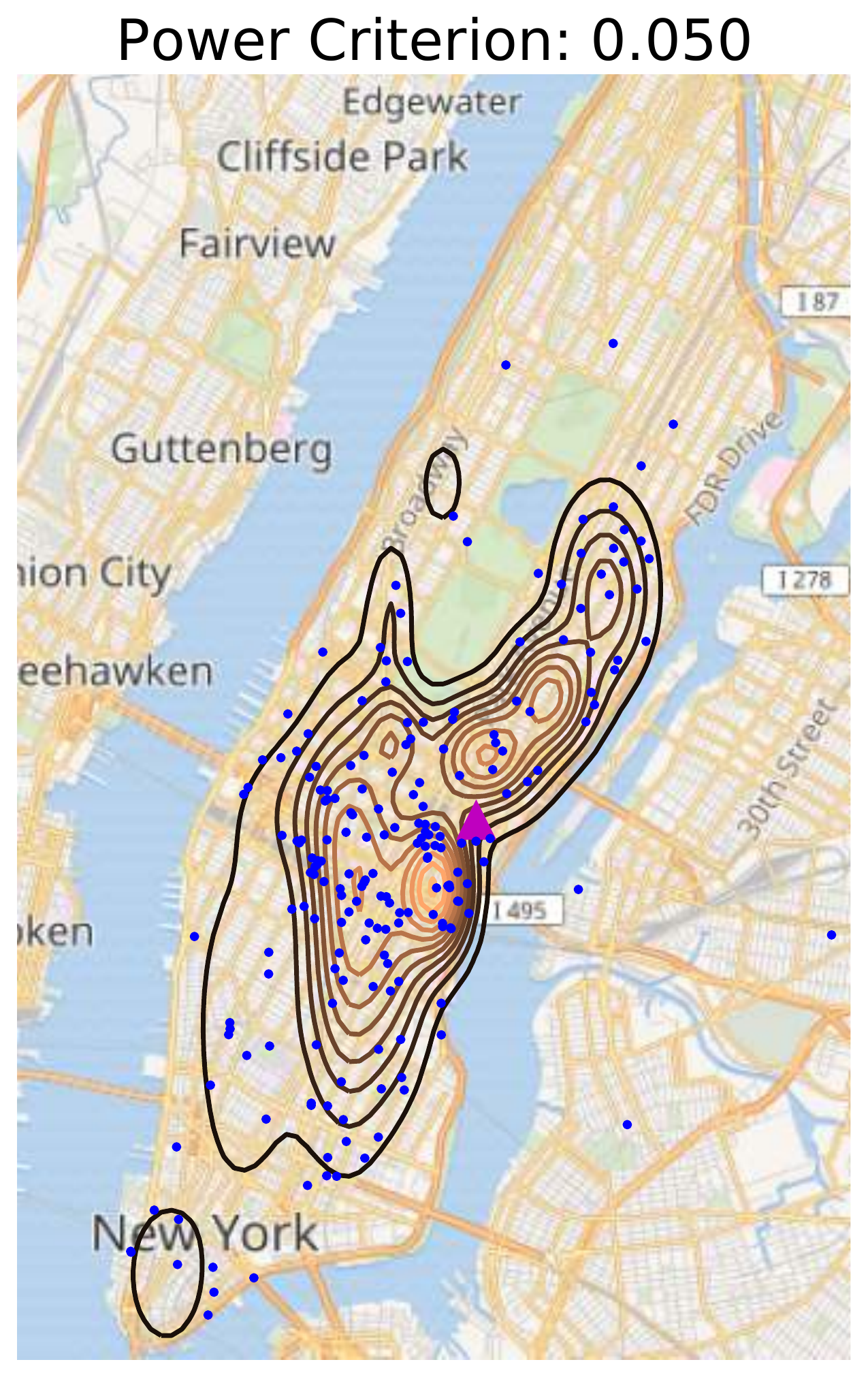}
\includegraphics[width=\imgw]{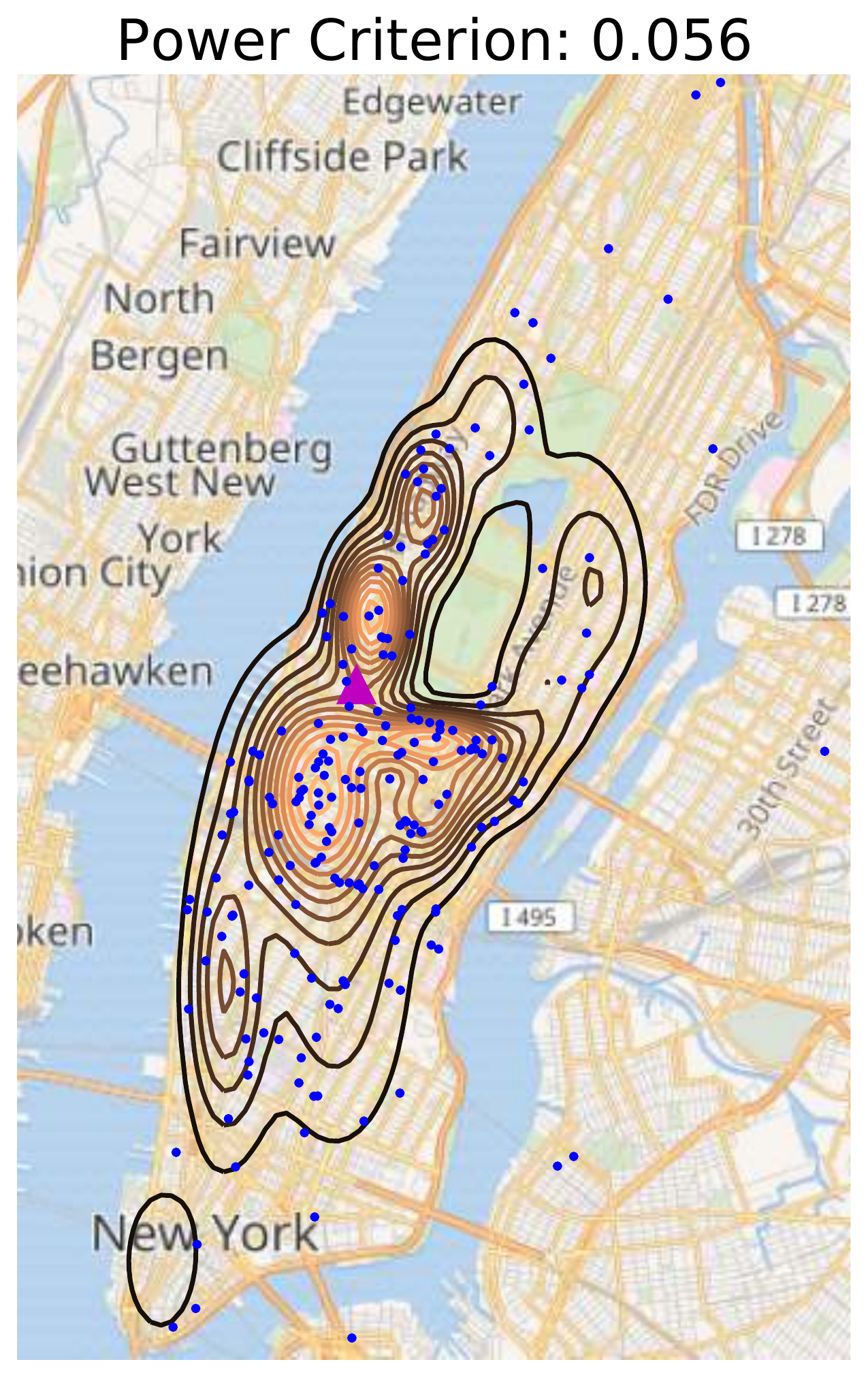}
\includegraphics[width=\imgw]{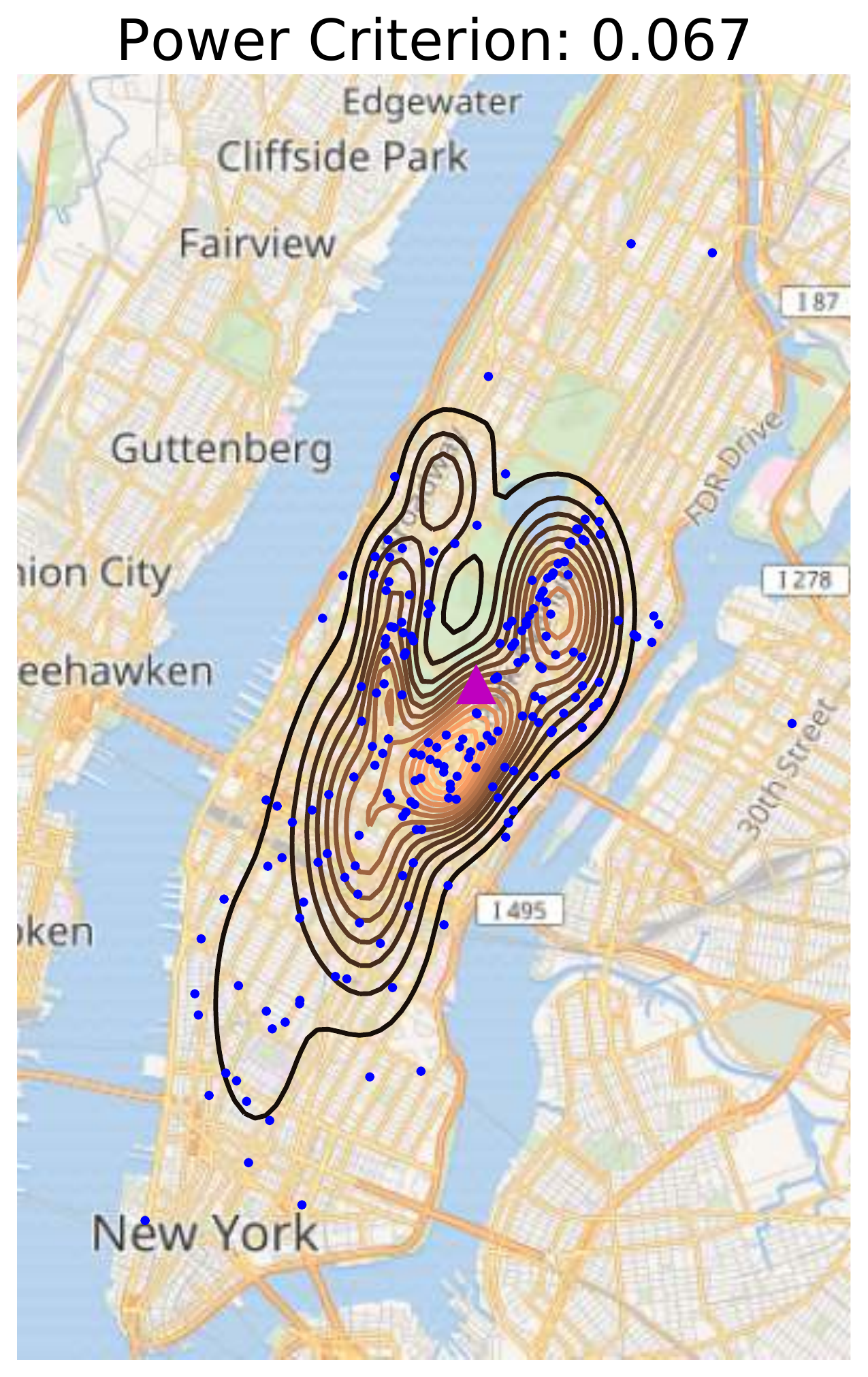}
\includegraphics[width=\imgw]{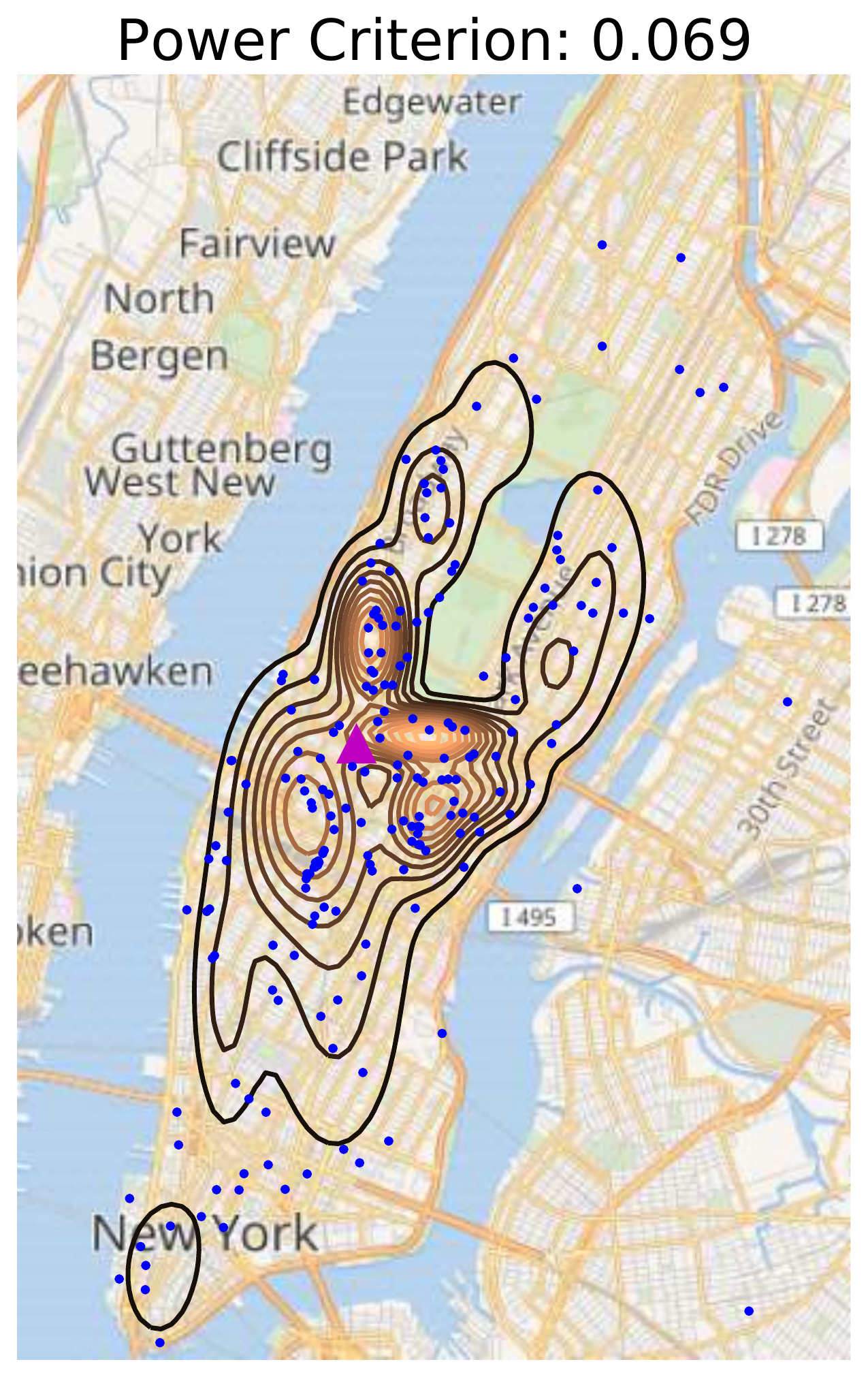}

We observe that the power criterion (dimensionless quantity) roughly
corresponds to the degree of mismatch between the conditional model
$p(\cdot|\mathbf{v})$ and the observed data i.e., high when the mismatch
is large. We note that power criterion values are affected by the
choice of the two kernels $k,l$, quality of the trained model, and
the sample size used to compute the power criterion. Changing the
two kernels may (and likely will) change the values of the power criterion,
and the ordering of these cases. Thoroughly studying the effects of
these factors on the computed power criterion will be an interesting
topic of future research.
\end{document}